\pdfoutput=1
\documentclass{article}

\PassOptionsToPackage{numbers}{natbib}
\usepackage[main, final]{neurips_2026}

\usepackage[utf8]{inputenc} %
\usepackage[T1]{fontenc}    %
\usepackage{hyperref}       %
\usepackage{url}            %
\usepackage{booktabs}       %
\usepackage{amsfonts}       %
\usepackage{nicefrac}       %
\usepackage{microtype}      %
\usepackage{xcolor}         %
\usepackage{amsmath}
\usepackage{amssymb}
\usepackage{mathtools}
\usepackage{amsthm}
\usepackage[capitalize,noabbrev]{cleveref}

\theoremstyle{plain}
\newtheorem{theorem}{Theorem}[section]
\newtheorem{proposition}[theorem]{Proposition}
\newtheorem{lemma}[theorem]{Lemma}
\newtheorem{corollary}[theorem]{Corollary}
\theoremstyle{definition}
\newtheorem{definition}[theorem]{Definition}

\theoremstyle{remark}
\newtheorem{remark}[theorem]{Remark}

\usepackage{xcolor, colortbl}
\colorlet{DistShade}{orange!8}
\colorlet{NStepShade}{cyan!8}
\colorlet{D4PGShade}{violet!8}
\colorlet{PureShade}{gray!6}
\colorlet{OursShade}{magenta!10}
\colorlet{GrayShade}{gray!6}
\colorlet{CyanShade}{cyan!8}
\colorlet{MagentaShade}{magenta!10}

\usepackage{graphicx}
\usepackage{algorithm}
\usepackage{algorithmic}
\usepackage{multirow}
\usepackage{caption}
\usepackage{subcaption}
\usepackage{makecell}
\usepackage{enumitem}
\usepackage{pifont}
\usepackage{wrapfig}
\usepackage{float}  
\usepackage{needspace}      %

\newcommand{\offon}[2]{\textcolor{gray}{#1}\,$\rightarrow$\,\textbf{#2}}
\newcommand{\pmv}[1]{\raisebox{0.2ex}{\scriptsize$\pm #1$}}
\newcommand{\cmark}{\ding{51}}
\newcommand{\xmark}{\ding{55}}

\title{Variance-Averse $n$-Step Offline Reinforcement Learning for Sparse Long-Horizon Environments}

\author{%
  \textbf{Guhyeon Kang}, \textbf{Minhae Kwon}\thanks{Corresponding author} \\
   Department of Electrical and Computer Engineering\\ Sungkyunkwan University, Republic of Korea \\
  \texttt{\{guhyeon.kang, minhae.kwon\}@skku.edu} \\
}

\begin{document}

\maketitle

\begin{abstract}
Generative actors are transforming offline reinforcement learning (RL) by enabling expressive policy classes that model complex action distributions. However, this expressiveness also exposes a key challenge in heterogeneous datasets: generative policies can reproduce \emph{unreliable action modes} whose return distributions exhibit high variance, occasionally yielding high returns by chance but lacking consistency. Consequently, maximizing the expected $Q$-value alone is insufficient for identifying \emph{reliable actions}.
We propose {\color{magenta}\textbf{VAN-Flow}} (\emph{Variance-Averse $n$-step Flow}), a framework that promotes reliable actions in generative offline RL. VAN-Flow combines (i) a categorical distributional critic, (ii) a \emph{variance-averse expectation} operator that smoothly reweights atom probabilities to favor actions with both high returns and low dispersion, and (iii) a flow-matching generative actor guided via rejection sampling. Unlike CVaR or mean--variance objectives, the operator redistributes probability mass over the categorical return distribution without hard truncation or auxiliary penalty terms.
Across more than $40$ tasks from D4RL and OGBench, VAN-Flow consistently outperforms strong baselines, with the largest gains in long-horizon and high-variance regimes where reliable action selection becomes critical.\footnote{Code: \url{https://github.com/9hyeon1225/van}}  %

\end{abstract}
\section{Introduction}\label{sec:introduction}
Generative actors, both diffusion-based~\cite{wang2023diffusion,chi2023diffusion} and flow-matching policies~\cite{park2025fql,li2025qchunking} are reshaping offline reinforcement learning (RL) and increasingly serve as the backbone of foundation models for robotics and physical AI. Their strength is representational: they capture multi-modal action distributions that Gaussian policies cannot. However, this very flexibility becomes a liability in offline RL. Offline datasets are typically collected from heterogeneous behavior policies of varying quality~\cite{10343100,10610308}, so the same state–action pair may yield drastically different outcomes, success in some trajectories and failure in others. A generative actor faithfully fits the dataset distribution and therefore equally represents both reliable and unreliable modes, including those that yield high returns by chance but lack consistency. Without explicit guidance, the actor may converge toward such unreliable modes~\cite{eom2024selective}, especially in sparse-reward, long-horizon tasks where successful trajectories are rare and even small amounts of unreliable behavior dominate value estimation. This raises a fundamental design question. 
\begin{center}
\textit{How can we guide a generative actor to selectively learn the reliable behaviors in an offline dataset?}
\end{center}

We argue that reliable behaviors are those with \emph{low return variance under the dataset}: a state–action pair whose returns are consistent across observed trajectories is more dependable than one that occasionally yields high returns by chance. A natural concern is whether this conflates reliability with mere safety, since a trivial action with consistently low returns is also low-variance. We address this by jointly favoring \emph{high} return and low variance, so the actor is steered toward modes that are simultaneously rewarding and dependable, not merely safe. Crucially, expected $Q$-value alone cannot capture reliability: two state–action pairs with identical $\mathbb{E}[Q]$ may differ drastically in dispersion, and a regression-based critic exposes only the mean. Distributional value estimation~\cite{bellemare2017distributional,dabney2018quantile,dabney2018iqn} addresses this gap by modeling the full return distribution.

To exploit this structure, we introduce a \textbf{\color{magenta}variance-averse expectation} $\mathcal{E}(\cdot)$, an aggregation operator over categorical return distributions that smoothly reweights atom probabilities to jointly favor high return and low variance. Rather than defining a new risk-sensitive objective, $\mathcal{E}(\cdot)$ is designed for action selection in generative offline RL, where a policy produces multiple candidate actions that must be evaluated using a distributional critic. 
While risk-sensitive RL has studied measures such as CVaR~\cite{dabney2018iqn}
and entropic risk to encode risk preferences in control problems,
our goal is different: improving \emph{reliability} when selecting
among candidate actions generated by a multimodal policy.
Unlike tail-truncation approaches such as CVaR, $\mathcal{E}(\cdot)$
smoothly redistributes probability mass across the categorical
support, favoring actions whose returns are both high and consistent.

\begin{wrapfigure}{t}{0.5\linewidth}
    \vspace{-1.5em}
    \centering
    \includegraphics[width=\linewidth]{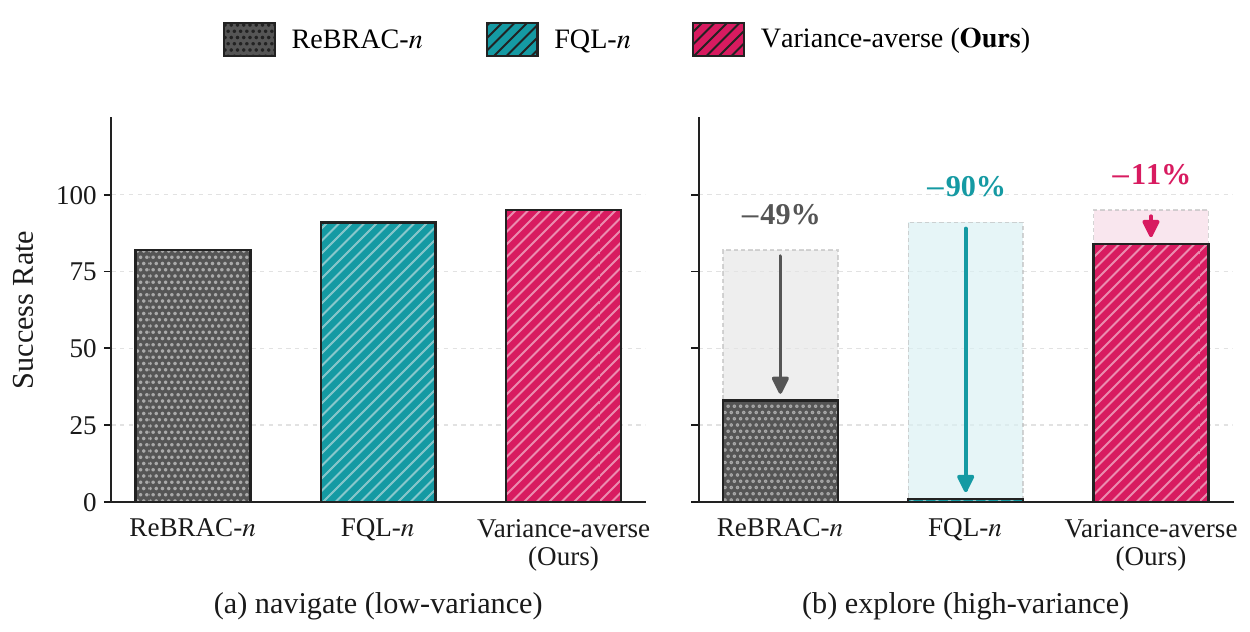}
    \caption{Success rate on \texttt{antmaze-large} under two dataset variance regimes: (a) \texttt{navigate} (low-variance) and (b) \texttt{explore} (high-variance). In (b), the shaded bar denotes the low-variance reference, and the downward arrow indicates the performance drop.}
    \label{fig:intro}
    \vspace{-1em}
\end{wrapfigure}

This framework is especially impactful in $n$-step learning. While $n$-step returns 
reduce bootstrapping bias~\cite{kearns2000bias,park2025horizon}, they amplify 
return variance under heterogeneous data~\cite{precup2000eligibility,sutton2018rl,ren2021nearly}, 
widening the gap between reliable and unreliable behaviors. Figure~\ref{fig:intro} 
confirms this: on the high-variance \texttt{explore} dataset, the flow-based method 
(FQL-n) suffers a markedly larger drop than its Gaussian counterpart (ReBRAC-n) under 
the same $n$-step setting, exposing the very vulnerability of generative actors 
discussed above. In contrast, our reliability-driven $n$-step training preserves 
performance with only marginal degradation.

We instantiate these ideas in \textbf{\color{magenta}VAN-Flow} (Variance-Averse $n$-step Flow), which integrates: (i) a categorical distributional critic that exposes return variance, (ii) the variance-averse expectation $\mathcal{E}(\cdot)$ for reliability scoring, and (iii) a flow-matching actor guided by $\mathcal{E}(\cdot)$ via rejection sampling, which re-ranks candidate actions so that the final selection aligns with the reliable subset of the dataset.

Our main contributions are as follows. 
1) We identify a fundamental vulnerability of generative actors in offline RL: under heterogeneous datasets, they tend to reproduce unreliable action modes whose return distributions exhibit high variance. We formalize \emph{reliable behavior} through the dispersion of dataset return distributions, showing that expected $Q$-values alone are insufficient for guiding generative policies. 
2) We propose a novel variance-averse expectation $\mathcal{E}(\cdot)$ for categorical return distributions, which smoothly reweights atom probabilities to jointly favor high-return and low-variance actions without hard truncation or auxiliary penalty terms. 
3) We develop \textbf{VAN-Flow}, a unified framework combining a flow-matching actor, a categorical distributional critic, and $\mathcal{E}(\cdot)$-guided rejection sampling. 
4) Across more than $40$ tasks from D4RL and OGBench, VAN-Flow consistently outperforms strong baselines, with the largest gains in high-variance and long-horizon regimes. Extensive ablations and sensitivity analyses further confirm the effectiveness of each component and highlight the importance of reliability-aware action selection in $n$-step offline RL. To the best of our knowledge, this is the first work to use variance-averse aggregation over categorical return distributions to select reliable actions for generative actors in offline RL, and to show empirically that this matters for stable $n$-step learning under heterogeneous data.

\section{Preliminaries}
\paragraph{Markov Decision Process.} 
We consider an MDP $\mathcal{M} = (\mathcal{S}, \mathcal{A}, P, r, \gamma)$ with state $\boldsymbol{s}_t \in \mathcal{S}$, action $\boldsymbol{a}_t \in \mathcal{A}$, transition $P$, reward $r$, and discount $\gamma \in [0,1)$. The objective is to learn a policy $\pi(\cdot \mid \boldsymbol{s}_t)$ that maximizes the expected return $\mathcal{J}(\pi) = \mathbb{E}_\pi[G_t^{(\infty)}] = \mathbb{E}_\pi[\sum_{k=t}^{\infty}\gamma^{k-t}r_k]$.

\paragraph{Offline RL with $n$-step Return.} 
In offline RL, learning relies on a fixed dataset~\cite{lee2023survey} $\mathcal{D} = \{(\boldsymbol{s}_t, \boldsymbol{a}_t, G_t^{(n)}, \boldsymbol{s}_{t+n})\}_{t=1}^{N}$, where the $n$-step return $G_t^{(n)}=\sum_{k=t}^{t+n-1}\gamma^{k-t}r_k$ approximates $G_t^{(\infty)}$. Recent work shows that $n$-step returns improve performance in sparse-reward settings~\cite{park2025leq,burns2023multifrequency}. In the actor--critic setup, the critic $Q_\psi$ is trained against the $n$-step Bellman target,
\begin{align}
\mathcal{T}^{(n)}Q_\psi(\boldsymbol{s}_{t},\boldsymbol{a}_t)=G_{t}^{(n)}+\gamma^n Q_\psi(\boldsymbol{s}_{t+n},\boldsymbol{a}_{t+n}^\pi),
\label{eq:nRL_Bellman}
\end{align}
with $\boldsymbol{a}_{t+n}^\pi \sim \pi_\theta(\cdot \mid \boldsymbol{s}_{t+n})$, and the actor maximizes $\mathbb{E}_{\boldsymbol{s}_t}[Q_\psi(\boldsymbol{s}_t, \boldsymbol{a}_t^\pi)]$.

\paragraph{Categorical Value Estimation.} 
A categorical critic $Z_\psi(\boldsymbol{s}_t, \boldsymbol{a}_t)$ models the full return distribution over $I$ fixed atoms $\{z_i\}_{i=1}^{I}$ with probabilities $\{p_i\}_{i=1}^{I}$, exposing variance structure that a regression-based critic cannot, a property we exploit in Section~\ref{sec:vae}. The expected return recovers as $\mathbb{E}[G_t^{(\infty)}] = \sum_i p_i z_i$. The critic is trained via cross-entropy against the $n$-step distributional Bellman target.
\begin{align}
\mathcal{T}_z^{(n)}Z_{\psi}(\boldsymbol{s}_t, \boldsymbol{a}_t)=\Phi\Big(\sum_{i=1}^{I} p_i(\boldsymbol{s}_{t+n}, \boldsymbol{a}_{t+n}^\pi)\,\delta_{G_t^{(n)}+\gamma^n z_i}\Big)
\label{eq:DnRL_BO}
\end{align}
Here, $p_i(\boldsymbol{s}, \boldsymbol{a})$ is the probability that the critic assigns to atom $z_i$ and $\delta_{y}$ is the Dirac measure at $y$ (not to be confused with the aversion coefficient $\delta$ introduced in Section~\ref{sec:vae}). In words, each atom $z_i$ is shifted to $G_t^{(n)}+\gamma^n z_i$ while keeping its probability mass, and the shifted distribution is projected back onto the fixed support by the projection operator $\Phi(\cdot)$. We employ the projection method introduced in~\cite{bellemare2017distributional} (see Appendix~\ref{app:C}).

\paragraph{Flow-matching Policy.} 
Flow matching parameterizes a stochastic policy through a state-conditioned velocity field $v^\pi_\theta(\tau, \boldsymbol{s}_t, \boldsymbol{x}_t(\tau))$ defining the ODE, 
\begin{equation}
\frac{d\boldsymbol{x}_t(\tau)}{d\tau} = v^\pi_\theta(\tau, \boldsymbol{s}_t, \boldsymbol{x}_t(\tau)), \quad \boldsymbol{x}_t(0)=\boldsymbol{\epsilon},
\label{eq:flow_ode}
\end{equation}
with $\boldsymbol{\epsilon} \sim \mathcal{N}(\mathbf{0},\mathbb{I})$. The trajectory linearly interpolates noise and action,
\begin{equation}
\boldsymbol{x}_t(\tau) = (1-\tau)\boldsymbol{\epsilon} + \tau \boldsymbol{a}_t,
\label{eq:flow_interpolation}
\end{equation}
and the action is taken as $\boldsymbol{a}_t = \boldsymbol{x}_t(1)$. The velocity field is trained by regressing onto $\boldsymbol{a}_t - \boldsymbol{\epsilon}$~\cite{park2025fql}.

\section{Variance-averse Expectation with Categorical Value Estimation}\label{sec:vae}
\begin{wrapfigure}{r}{0.55\linewidth}
    \vspace{-0.5cm}
    \centering
\includegraphics[width=\linewidth]{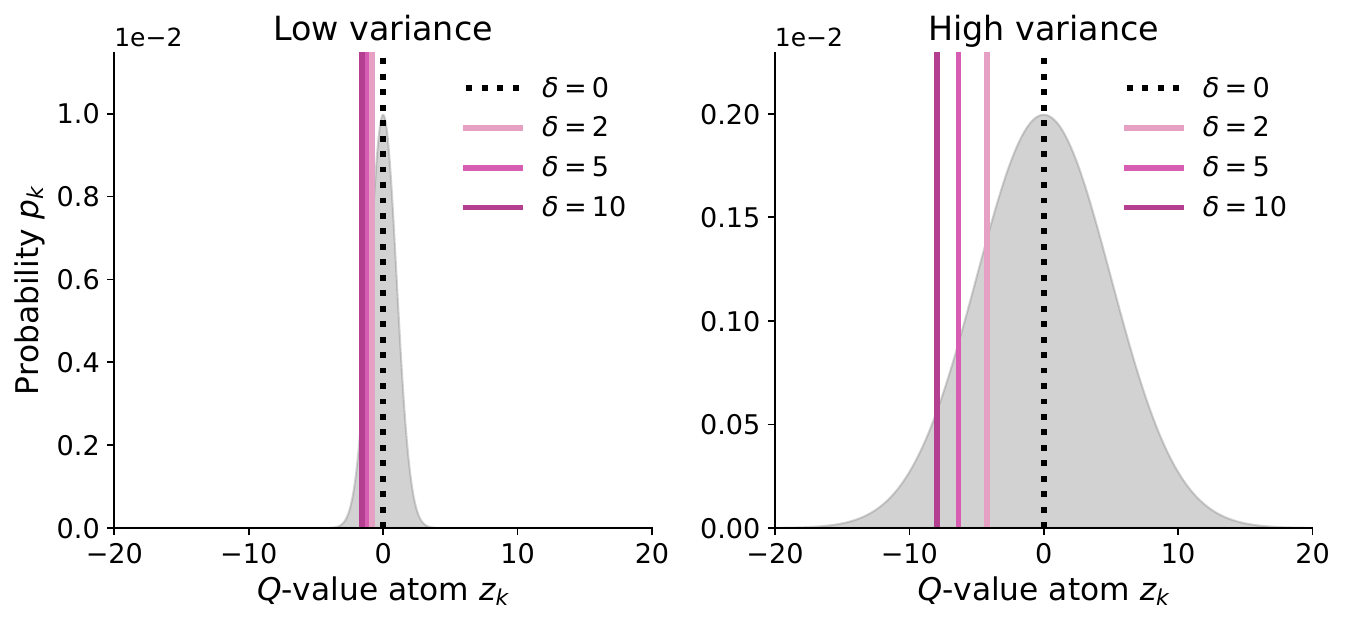}
    \caption{Example of variance-averse return with varying $\delta$ and return variance.}
    \label{fig:VaE}
    \vspace{-1.5em}
\end{wrapfigure}

The key contribution of this work is the proposal of a \textbf{\color{magenta}variance-averse expectation} that penalizes high-variance return distributions, where variance refers to the dispersion of the return distribution captured by the categorical critic. Such variance would otherwise undermine the effectiveness of the $n$-step method.

\begin{definition}[Variance-averse Expectation $\mathcal{E}(Z)$] \label{def:VaE}
    Let $Z$ be a categorical return distribution with $\mathbb{P}(Z=z_i)=p_i$ over ordered atoms $z_1<\cdots<z_I$. The {\color{magenta}variance-averse expectation $\mathcal{E}(Z)$} is defined as follows.
\begin{equation}
\mathcal{E}(Z)
= \sum_{i=1}^{I}\underbrace{\left(\frac{p_i\Big(1-\mathcal{C}\left(z_i\right)\Big)^{\delta}}
{\sum_{j=1}^{I}p_j\left(1-\mathcal{C}\left(z_j\right)\right)^\delta}\right)}_{\color{magenta}\text{variance-averse probability}}z_i \label{eq:VaE}
\end{equation}
\end{definition}
In~\eqref{eq:VaE}, $\mathcal{C}(z_i) = \mathbb{P}(Z \le z_i)$ denotes the cumulative distribution function (CDF) evaluated at $z_i$, and $\delta \ge 0$ controls the degree of aversion. As $\delta$ increases, the variance-averse expectation increasingly penalizes high-variance return distributions. When $\delta = 0$, the operator reduces to the standard expectation, i.e., $\mathcal{E}(Z) = \mathbb{E}[Z]$. An illustrative example of $\mathcal{E}(Z)$ for different values of $\delta$ is presented in Figure~\ref{fig:VaE}.\footnote{We discuss design considerations including the role of the support boundary and numerical stability in Appendix~\ref{app:b2}.}

While the variance-averse expectation $\mathcal{E}(Z)$ in Definition~\ref{def:VaE} focuses on the lower-return regions of the distribution, its underlying role is to penalize high-variance return distributions. We formalize this through the convex order, denoted by $\le_{\mathrm{cx}}$, which provides a distributional notion of dispersion beyond second-order moments. Specifically, for two return distributions $Z^{\dagger}$ and $Z^{\circ}$ with identical expectations, the relation $Z^{\dagger} \le_{\mathrm{cx}} Z^{\circ}$ indicates that $Z^{\circ}$ is more dispersed than $Z^{\dagger}$, meaning that its probability mass is spread further away from the mean. The operator in~\eqref{eq:VaE} is a finite-atom discretization of a \emph{spectral} functional of the quantile function of $Z$. The following theorem establishes variance aversion exactly for that functional and bounds its gap to the implemented operator.

\begin{theorem}[Variance aversion of $\mathcal{E}$]
\label{thm:VaE_mps}
Let $Z$ be a categorical return distribution as in Definition~\ref{def:VaE}, write $\bar{\mathcal{C}}(z_i)=1-\mathcal{C}(z_i)$ for its survival function with $\bar{\mathcal{C}}(z_0)=1$, and define the spectral form
\begin{equation}
\mathcal{E}^{\mathrm{sp}}(Z)
=\sum_{i=1}^{I}\Big[\bar{\mathcal{C}}(z_{i-1})^{\delta+1}-\bar{\mathcal{C}}(z_i)^{\delta+1}\Big]\,z_i
=(\delta+1)\int_0^1 \Omega_Z(u)\,(1-u)^{\delta}\,du,
\label{eq:VaE_spectral}
\end{equation}
where $\Omega_Z$ is the quantile function of $Z$. Then, for any $\delta>0$:
\begin{enumerate}[label=(\roman*),leftmargin=*,nosep]
\item \textbf{(Convex-order monotonicity)} If $\mathbb{E}[Z^{\dagger}]=\mathbb{E}[Z^{\circ}]$ and $Z^{\dagger}\le_{\mathrm{cx}}Z^{\circ}$, then $\mathcal{E}^{\mathrm{sp}}(Z^{\circ})\le\mathcal{E}^{\mathrm{sp}}(Z^{\dagger})$.
\item \textbf{(Discretization error)} The operator $\mathcal{E}$ in~\eqref{eq:VaE} is the normalized right-endpoint discretization of~\eqref{eq:VaE_spectral}, and with $R=z_I-z_1$,
\begin{equation}
\big|\mathcal{E}(Z)-\mathcal{E}^{\mathrm{sp}}(Z)\big|
\;\le\; R\,\varepsilon_\delta(Z),
\qquad
\varepsilon_\delta(Z):=1-(\delta+1)\sum_{i=1}^{I}p_i\,\bar{\mathcal{C}}(z_i)^{\delta}\in[0,1],
\label{eq:VaE_gap}
\end{equation}
where $\varepsilon_\delta(Z)\le(\delta+1)\max_i p_i$ and, for $\delta\ge1$, $\varepsilon_\delta(Z)\le\delta(\delta+1)\sum_{i}p_i^{2}$. Consequently, under the assumptions of (i),
\begin{equation}
\mathcal{E}(Z^{\circ})\;\le\;\mathcal{E}(Z^{\dagger})+R\big[\varepsilon_\delta(Z^{\dagger})+\varepsilon_\delta(Z^{\circ})\big].
\label{eq:VaE_transfer}
\end{equation}
If $Z=\Phi(Y)$ is the C51 projection of a target $Y$ supported on $[V_{\min},V_{\max}]$ with a Lipschitz-continuous CDF, then $\max_i p_i=O(1/I)$ and $\varepsilon_\delta(Z)=O(1/I)$, so the gap vanishes as the number of atoms $I$ increases.
\end{enumerate}
\end{theorem}
Part (i) states that the spectral form prefers the less dispersed of two equal-mean return distributions, and part (ii) states that the implemented operator inherits this preference up to a gap governed by how concentrated the critic's mass is on individual atoms. In particular, whenever the convex-order margin $\mathcal{E}^{\mathrm{sp}}(Z^{\dagger})-\mathcal{E}^{\mathrm{sp}}(Z^{\circ})$ exceeds the error budget in~\eqref{eq:VaE_transfer}, $\mathcal{E}$ induces the same ordering as the spectral form (Corollary~\ref{cor:ordering}). The certificate $\varepsilon_\delta(Z)$ is computable from the critic output. On trained critics the realized gap is small (median $0.37\%$ of $R$), and $\mathcal{E}$ and $\mathcal{E}^{\mathrm{sp}}$ select the same rejection-sampling candidate in $96.9\%$ of states with Kendall's $\tau=0.99$ (Appendix~\ref{app:b3}). Proofs are provided in Appendices~\ref{app:b1} and~\ref{app:b3}.

\section{VAN-Flow: Variance-averse $n$-step RL with Flow-based Policy}
\begin{wrapfigure}{R}{0.5\linewidth}  %
\vspace{-3\intextsep}                  %
\begin{minipage}{\linewidth}
\begin{algorithm}[H]
\caption{\textbf{\color{magenta}VAN-Flow}}
\label{alg:main}
\footnotesize
\begin{algorithmic}[1]
\STATE \textbf{Input:} Dataset $\mathcal{D}$, networks $Z_\psi$, $v^\pi_\theta$
\STATE \textbf{Hyperparameters:} discount $\gamma$, return-step $n$, aversion intensity $\delta$, num of candidates $M$, flow steps $K$, flow-matching $\lambda$, learning rates $\alpha_{Z}, \alpha_{\pi}$, total epochs $E$
\FOR{$e = 1$ to $E$}
    \STATE Sample mini-batch $B$ from $\mathcal{D}$
    \STATE Generate target action $\boldsymbol{a}_{t+n}^*$ via rejection sampling~\eqref{eq:rejection_sampling}
    \STATE Compute target $\mathcal{T}_z^{(n)}Z_\psi$ by~\eqref{eq:proposed_criticBB}
    \STATE Compute critic loss $\mathcal{L}(\psi)$ by~\eqref{eq:pro_critic_loss}
    \STATE Compute actor loss $\mathcal{L}(\theta)$ by~\eqref{eq:pro_actor_loss}
    \STATE $\psi\leftarrow\psi-\alpha_Z\nabla_\psi\mathcal{L}(\psi)$
    \STATE $\theta\leftarrow\theta-\alpha_\pi\nabla_\theta\mathcal{L}(\theta)$
\ENDFOR
\end{algorithmic}
\end{algorithm}
\end{minipage}
\vspace{-1.5\intextsep}                %
\end{wrapfigure}

Building on the variance-averse expectation in Definition~\ref{def:VaE}, we propose \textbf{\color{magenta}VAN-Flow}, a flow-based offline RL framework that operationalizes \emph{reliable-action selection} for generative actors and exploits the benefits of $n$-step returns even under high-variance datasets. The flow-based architecture is motivated by the strong representational capacity of flow matching~\cite{chen2024deep,zhang2025energyflow}, which helps mitigate OOD action generation in offline RL~\cite{kang2023diffusion,park2025fql}. Furthermore, we introduce a variance-averse value--guided flow-matching policy that effectively leverages the proposed variance-averse expectation. In this section, we first describe the design of the flow-based actor and then provide implementation details of the proposed method.
\needspace{18\baselineskip}   %
\subsection{Flow-based Actor Design}
One distinctive feature of our flow-based actor is the use of rejection sampling based on the variance-averse expectation. In this subsection, we describe the action construction procedure and the corresponding rejection sampling mechanism.

\paragraph{Flow-based Action Construction.}
We design an action generation mechanism that directly exploits the learned velocity field. Rather than explicitly solving the continuous-time ODE in~\eqref{eq:flow_ode}, our approach constructs actions via a lightweight discretization of the flow trajectory. Specifically, given the interpolated point $\boldsymbol{x}_t(\tau)$ defined in~\eqref{eq:flow_interpolation}, the policy evaluates the velocity field $v^\pi_\theta(\tau, \boldsymbol{s}_t, \boldsymbol{x}_t(\tau))$ conditioned on the current state $\boldsymbol{s}_t$. The action $\boldsymbol{a}_t^\pi$ is then obtained by advancing $\boldsymbol{x}_\tau$ along the predicted flow using an Euler update.
\begin{equation}
\boldsymbol{a}^\pi_t
=\boldsymbol{x}_\tau+\frac{v^\pi_\theta(\tau, \boldsymbol{s}_t, \boldsymbol{x}_t(\tau))}{K}\label{eq:EUler_step}
\end{equation}

\paragraph{Rejection Sampling.} While a flow-based actor faithfully reproduces dataset actions, it equally reproduces both reliable and unreliable modes, and may further generate OOD actions during early training due to an incompletely learned vector field. To selectively retain reliable, in-distribution candidates, we design the flow-based actor to perform rejection sampling guided by the variance-averse expectation~\cite{chen2023highfidelity,robert2004montecarlo}. Specifically, the actor generates $M$ candidate actions $\mathbf{A}_t^\pi = \{\boldsymbol{a}_{t,1}^\pi, \ldots, \boldsymbol{a}_{t,m}^\pi, \ldots, \boldsymbol{a}_{t,M}^\pi \}$, and selects the action $\boldsymbol{a}_{t}^*$ that maximizes the variance-averse value as follows. 
\begin{align}
    \boldsymbol{a}_t^* = \arg\max_{\boldsymbol{a}_{t,m}^\pi\in\mathbf{A}_t^\pi} {\color{magenta}\mathcal{E}\Big(Z_\psi\left(\boldsymbol{s}_{t},\boldsymbol{a}_{t,m}^\pi\right)\Big)} \label{eq:rejection_sampling}
\end{align}
Such rejection sampling enables the actor to select actions with low return variance from in-distribution samples.

\subsection{Variance-averse $Q$-guided Flow Matching}
In this subsection, we provide the implementation details of \textbf{\color{magenta}VAN-Flow}.
The proposed method is implemented within an actor--critic framework. Our framework captures return variance through a categorical critic $Z_\psi(\boldsymbol{s}_t, \boldsymbol{a}_t)$ and employs a flow-based actor $\pi_\theta(\cdot \mid \boldsymbol{s}_t)$.

\paragraph{Categorical Critic Training with $n$-step Return.}
The categorical critic is trained with the standard cross-entropy loss between the predicted return distribution and the projected $n$-step Bellman target. 
\begin{align}
&\mathcal{L}(\psi) = \mathbb{E}_{(\boldsymbol{s}_t,\boldsymbol{a}_t,G_t^{(n)},\boldsymbol{s}_{t+n})\sim\mathcal{D}} \Big[\mathrm{CE}\big(Z_\psi(\boldsymbol{s}_t,\boldsymbol{a}_t),\, \mathcal{T}_z^{(n)}Z_\psi(\boldsymbol{s}_t,\boldsymbol{a}_t)\big)\Big]
\label{eq:pro_critic_loss}\\
&\mathcal{T}_z^{(n)}Z_{\psi}(\boldsymbol{s}_t, \boldsymbol{a}_t) = \Phi\Big(\sum_{i=1}^{I} p_i(\boldsymbol{s}_{t+n}, \boldsymbol{a}_{t+n}^*)\,\delta_{G_t^{(n)}+\gamma^n z_i}\Big)
\label{eq:proposed_criticBB}
\end{align}
The notable component is that we use the variance-averse target action $\boldsymbol{a}_{t+n}^*$ obtained via rejection sampling~\eqref{eq:rejection_sampling} instead of an actor-sampled action $\boldsymbol{a}_{t+n}^\pi$. This distinct target action selection enables valid value propagation aligned with the training objective. Projection details follow C51~\cite{bellemare2017distributional} (see Appendix~\ref{app:C}).

\paragraph{Flow-based Actor with $Q$ Guidance.}
To train the actor network $\pi_\theta(\cdot \mid \boldsymbol{s}_t)$, we define 
a variance-averse, critic-guided loss function as follows.
\begin{align}
    \mathcal{L}(\theta) = \mathbb{E}_{\substack{(\boldsymbol{s}_t,\boldsymbol{a}_t) \sim \mathcal{D}, \\ \boldsymbol{\epsilon} \sim \mathcal{N}(0,\mathbb{I}), \\ \tau \sim \mathrm{Unif}[0,1]}} \Big[\lambda\big\|v^\pi_\theta(\tau, \boldsymbol{s}_t, \boldsymbol{x}_t(\tau)) 
    -(\boldsymbol{a}_t - \boldsymbol{\epsilon})\big\|_2^2 -\underbrace{{\color{magenta}\mathcal{E}\Big(Z_\psi\big(\boldsymbol{s}_t,\boldsymbol{x}_t(\tau+\Delta\tau)\big)\Big)}}_{Q\text{ guidance}} \Big] \label{eq:pro_actor_loss}
\end{align}
In~\eqref{eq:pro_actor_loss}, $\lambda \in \mathbb{R}^+$ is a constant that 
controls the magnitude of the flow-matching loss, and $\Delta\tau = 1/K$ 
denotes the flow integration step size, where $K$ is the total number of flow 
steps. The term $\boldsymbol{x}_t(\tau+\Delta\tau)$ denotes a single Euler 
step from the current flow state $\boldsymbol{x}_t(\tau)$, defined as follows.
\begin{align}
    \boldsymbol{x}_t(\tau+\Delta\tau)= \boldsymbol{x}_t(\tau)+\Delta\tau \cdot v^\pi_\theta(\tau, \boldsymbol{s}_t, \boldsymbol{x}_t(\tau))
\end{align}
Here, $\boldsymbol{x}_t(\tau+\Delta\tau)$ serves as a single-step action proxy 
through which the critic gradient is propagated to the velocity field, acting 
as a computationally efficient surrogate for the fully integrated action. 
Unlike prior work that employs full ODE integration for the 
actor~\cite{wang2023diffusion,10423845}, we reduce the ODE computation 
cost by using a single Euler step.

Under the actor loss in~\eqref{eq:pro_actor_loss}, the flow-based actor is trained to favor actions with low return variance, thereby mitigating training instability caused by high-variance $n$-step returns. The pseudocode of the proposed method is presented in Algorithm~\ref{alg:main}. For implementation, we also use a target network, as commonly employed for training stabilization~\cite{timothy2025ddpg,fujimoto2019off,NEURIPS2020_c16a53d5}, along with the double $Q$ method to mitigate overestimation bias~\cite{fujimoto2018addressing,vanhasselt2016double,haarnoja2018soft}.

\section{Simulation Result}\label{sec:experiments}
We conduct extensive experiments to evaluate \textbf{\color{magenta}VAN-Flow} on sparse, long-horizon reinforcement learning benchmarks. All results are reported as averages over five random seeds. Experiments are performed on a system equipped with an AMD Ryzen Threadripper PRO 9975WX (CPU) and an NVIDIA RTX PRO 6000 (GPU).
\subsection{Simulation Setup}

\paragraph{Benchmark.}
To evaluate our method under the core challenges of offline RL in sparse, long-horizon settings, we focus on two benchmark tasks that are widely regarded as difficult due to sparse rewards and long-horizon structure, and that additionally exhibit substantial return-variance changes across dataset variants. A detailed explanation of each benchmark task is presented in Appendix~\ref{app:D}.
\begin{itemize}[leftmargin=10pt, nosep]
    \item \textbf{D4RL~\cite{fu2020d4rl}:} We evaluate our method on the \texttt{AntMaze} tasks from D4RL, which remain challenging due to their long-horizon structure and sparse rewards. While many D4RL benchmarks exhibit performance saturation~\cite{tarasov2023corl,rafailov2024d5rl}, \texttt{AntMaze} continues to expose limitations in long-term credit assignment and value extrapolation.
    \item \textbf{OGBench~\cite{park2025ogbench}:} OGBench is a recently introduced benchmark designed to evaluate offline RL under long-horizon control, sparse and deceptive rewards, and limited dataset coverage. We focus on state-based OGBench tasks with noisy and suboptimal datasets, where sensitivity to out-of-distribution actions is pronounced, providing a stringent test for behavior-regularized, value-guided methods.
\end{itemize}

\paragraph{Baseline.} We compare our proposed \textbf{\color{magenta}VAN-Flow} under the following evaluation criteria. Detailed descriptions of the baseline algorithms are provided in Appendix~\ref{app:f1}.
\begin{itemize}[leftmargin=10pt, nosep]
    \item \textbf{Policy Architecture:} We evaluate methods employing \textbf{\textit{Gaussian-based}} policies (IQL~\cite{kostrikov2022iql}, ReBRAC~\cite{tarasov2023revisiting}, HIQL~\cite{park2023hiql}), which are the conventional approach for representing stochastic policies. In addition, we compare with methods using \textbf{\textit{flow-based}} policies (FQL~\cite{park2025fql}, BFN~\cite{li2025qchunking}, FQL-n, BFN-n~\cite{li2025qchunking}, QC~\cite{li2025qchunking}), which are most closely related to our actor design.
    \item \textbf{Algorithm Design:} For an extensive comparison, we evaluate $n$-step return methods (Retrace($\lambda$)~\cite{munos2016safe}, PQL~\cite{kozuno2021revisiting}, LEQ~\cite{park2025leq}, TD3BC+MS~\cite{peng2024deadly}, QC~\cite{li2025qchunking})\footnote{Retrace($\lambda$) and PQL were originally proposed for off-policy online RL. In our experiments, we adapt them to the fixed offline-dataset setting.} and distributional critics (PA-RL~\cite{mark2025policyagnostic}). We also include D4PG~\cite{maron2018distributional}, which incorporates both a distributional critic and $n$-step return, with an added behavior cloning (BC) term~\cite{fujimoto2021minimalist}. In addition, pure offline RL methods (IQL~\cite{kostrikov2022iql}, ReBRAC~\cite{tarasov2023revisiting}, FQL~\cite{park2025fql}) are considered as baselines that use $1$-step returns.

\end{itemize}

\subsection{Performance Comparison}

In this subsection, we present an extensive performance comparison of the proposed method.\footnote{We use results reported in~\cite{park2025ogbench,tarasov2023corl,li2025qchunking} along with our own experiments. Additional experiments are conducted only for tasks not covered in the original papers. Details are provided in Appendix~\ref{app:f2}.} Learning curves, the full performance table, and additional comparison results are provided in Appendices~\ref{app:g1}, \ref{app:g2}, and \ref{app:g3}, respectively. Furthermore, Off-to-online fine-tuning results are reported in Appendix~\ref{app:offon}.

\begin{table*}[t]
\centering
\footnotesize
\renewcommand{\arraystretch}{0.75}
\setlength{\tabcolsep}{4.5pt}
\caption{Performance comparison on OGBench. Each entry averages five tasks per environment. Best results are in bold.}
\label{tab:task_wise_results}
\resizebox{\textwidth}{!}{%
\begin{tabular}{l cccc ccccc !{\color{magenta!40}\vrule width 0.8pt} >{\columncolor{MagentaShade}}c}
\toprule
& \multicolumn{4}{c}{\textit{Gaussian-based}}
& \multicolumn{5}{c}{\textit{Flow-based}}
& \multicolumn{1}{c}{\textbf{Proposed}} \\
\cmidrule(lr){2-5} \cmidrule(lr){6-10} \cmidrule(lr){11-11}
\textbf{Task}
& IQL & HIQL & ReBRAC & ReBRAC-$n$
& FQL & BFN & FQL-$n$ & BFN-$n$ & QC
& VAN-Flow \\
\midrule
\multicolumn{11}{l}{\textit{Standard tasks}} \\[-2pt]
\midrule
\texttt{antmaze-large-navigate}     & 53 & 91 & 81 & 82 & 79 & 71 & 91 & 81 & 49 & \textbf{95} \\
\texttt{antmaze-teleport-navigate}  & 23 & 42 &  7 & 35 & 26 & 40 & 22 & 41 & 40 & \textbf{57} \\
\texttt{scene-play}                 &  0 & 38 & 11 & 40 & 57 & 85 & 18 & 57 & 84 & \textbf{93} \\
\texttt{puzzle-3x3-play}            &  0 & 12 & 55 & 10 & \textbf{100} & 98 & 98 & 58 & \textbf{100} & \textbf{100} \\
\midrule
\multicolumn{11}{l}{\textit{Long-horizon tasks}} \\[-2pt]
\midrule
\texttt{humanoidmaze-large-navigate} & 2 & 49 &  2 &  4 &  4 & 26 & 20 & 46 & 16 & \textbf{81} \\
\texttt{antmaze-giant-navigate}      & 4 & 65 & 26 & 16 &  9 & 11 & 54 &  5 &  2 & \textbf{68} \\
\texttt{humanoidmaze-giant-navigate} & 1 & 12 &  0 &  0 &  1 &  0 &  3 & 74 & 12 &\textbf{92} \\
\midrule
\multicolumn{11}{l}{\textit{Noisy tasks}} \\[-2pt]
\midrule
\texttt{antmaze-large-explore}      & 15 &  4 & 48 & 33 & 20 & 22 &  1 & 17 & 58 & \textbf{84} \\
\texttt{puzzle-3x3-noisy}           & 33 & 51 &  2 & 10 & \textbf{100} & \textbf{100} & 33 & 98 & \textbf{100} & 97 \\
\bottomrule
\end{tabular}%
}
\vspace{-0.5mm}
\end{table*}
\begin{table*}[t]
\centering
\footnotesize
\renewcommand{\arraystretch}{1.15}
\setlength{\tabcolsep}{3.5pt}
\caption{Offline performance on D4RL AntMaze tasks. Scores are normalized following the D4RL protocol. Entries marked ``--'' indicate results not reported in prior work. Best results are in bold.}
\label{tab:antmaze_results}
\resizebox{\textwidth}{!}{%
\begin{tabular}{l ccc ccccc c c !{\color{magenta!40}\vrule width 0.8pt} >{\columncolor{MagentaShade}}c}
\toprule
& \multicolumn{3}{c}{\textit{Pure Offline RL}}
& \multicolumn{5}{c}{\textit{$n$-step}}
& \textit{Dist.}
& \textit{Dist.+$n$}
& \multicolumn{1}{c}{\textbf{Proposed}} \\
\cmidrule(lr){2-4} \cmidrule(lr){5-9} \cmidrule(lr){10-10} \cmidrule(lr){11-11} \cmidrule(lr){12-12}
\textbf{Task}
& IQL & ReBRAC & FQL
& Retrace($\lambda$) & PQL & LEQ & QC & TD3BC+MS
& PA-RL
& D4PG
& VAN-Flow \\
\midrule
\texttt{umaze}
& 77.0\pmv{5.5} & 97.7\pmv{1.5} & 96.0\pmv{2.0}
& 96.0\pmv{3.0} & 94.0\pmv{3.0} & 94.4\pmv{6.3} & 95.0\pmv{1.4} & --
& -- & 87.6\pmv{4.1}
& \textbf{98.0}\pmv{2.8} \\

\texttt{umaze-diverse}
& 54.2\pmv{5.5} & 83.5\pmv{7.0} & 89.0\pmv{2.0}
& 84.0\pmv{10.0} & 63.0\pmv{6.0} & 71.0\pmv{12.3} & 78.0\pmv{8.5} & --
& -- & 61.3\pmv{9.4}
& \textbf{95.6}\pmv{2.6} \\

\texttt{medium-play}
& 65.7\pmv{11.7} & 89.5\pmv{3.3} & 78.0\pmv{7.0}
& 79.0\pmv{10.0} & 81.0\pmv{6.0} & 58.8\pmv{33.0} & 72.0\pmv{2.4} & 70.0
& 88.0 & 89.1\pmv{3.0}
& \textbf{90.7}\pmv{3.1} \\

\texttt{medium-diverse}
& 73.7\pmv{5.4} & 83.5\pmv{8.2} & 71.0\pmv{13.0}
& 53.0\pmv{17.0} & 60.0\pmv{33.0} & 46.2\pmv{23.2} & 64.0\pmv{2.8} & 66.0
& 88.0 & 81.3\pmv{4.1}
& \textbf{89.3}\pmv{5.0} \\

\texttt{large-play}
& 42.0\pmv{4.5} & 52.2\pmv{29.0} & 84.0\pmv{7.0}
& 12.0\pmv{26.0} & 33.0\pmv{10.0} & 58.6\pmv{9.1} & 80.0\pmv{2.8} & 22.0
& \textbf{87.0} & 71.3\pmv{4.1}
& 86.5\pmv{5.7} \\

\texttt{large-diverse}
& 30.2\pmv{3.6} & 64.0\pmv{5.4} & 83.0\pmv{4.0}
& 20.0\pmv{3.0} & 44.0\pmv{9.0} & 60.2\pmv{18.3} & 82.0\pmv{5.6} & 28.0
& 73.0 & 80.6\pmv{6.1}
& \textbf{88.6}\pmv{3.0} \\
\bottomrule
\end{tabular}%
}
\vspace{-0.4cm}
\end{table*}

\paragraph{Comparison of Policy Architecture.}
Table~\ref{tab:task_wise_results} compares different policy architectures across tasks with varying horizon lengths and noise levels. Overall, \textbf{\textit{Gaussian-based}} policies (IQL, ReBRAC, HIQL) exhibit strong performance on several tasks but show noticeable degradation in long-horizon and noisy environments, reflecting limited expressiveness under complex multi-modal action distributions.

\textbf{\textit{Flow-based}} policies substantially improve performance in these challenging settings, demonstrating the advantage of expressive action modeling; however, naive flow-based methods (FQL, BFN) remain unstable and sensitive to distributional shift.
While incorporating $n$-step returns into flow-based policies (FQL-n, BFN-n) yields more robust performance, these methods still suffer from performance drops in high-variance regimes.
In contrast, \textbf{\color{magenta}VAN-Flow} consistently achieves strong results across long-horizon and noisy tasks, indicating that combining expressive flow-based policies with variance-aware value guidance is crucial for stable and effective offline RL. In particular, on high-variance datasets such as \texttt{antmaze-large-explore}, flow-based $n$-step methods (FQL-$n$, BFN-$n$) collapse far more severely than their Gaussian counterpart (ReBRAC-$n$) under the same $n$-step setting, while VAN-Flow remains robust.

\paragraph{Comparison of Algorithmic Design.} Table~\ref{tab:antmaze_results} compares representative design choices on D4RL maze tasks. Pure offline RL methods perform competitively on small-scale tasks but degrade as the horizon and dataset diversity increase. Methods using only $n$-step returns (Retrace($\lambda$), PQL, LEQ, QC, TD3BC+MS) show inconsistent performance across task scales and often drop sharply on large or diverse settings. PA-RL, which uses a distributional critic without $n$-step returns, is strong on the medium mazes but falls behind on \texttt{large-diverse}. Even D4PG, which combines both components, fails to maintain robust performance as task difficulty increases. In contrast, \textbf{\color{magenta}VAN-Flow} consistently achieves strong results from \texttt{umaze} to \texttt{large-diverse}, indicating that \emph{explicit return-variance control} is crucial for stably leveraging $n$-step returns in offline RL.

\subsection{Ablation Studies}
We conduct controlled ablations to isolate the effect of each major
component while keeping all other components fixed. Specifically,
we examine (1) the distributional critic and flow-based policy,
(2) the variance-averse expectation, and (3) the $Q$ guidance mechanism.
For clarity, we report representative results on one task from each
environment. Additional component-wise ablations that further isolate
the contributions of variance-averse expectation, the distributional critic, and the flow-based policy across action-dimension regimes are provided in Appendix~\ref{app:component_ablation}.

\paragraph{Categorical Critic and Flow-based Actor.}
\begin{wrapfigure}{r}{0.5\linewidth}
    \vspace{-\intextsep}
    \centering
    \includegraphics[width=0.8\linewidth]{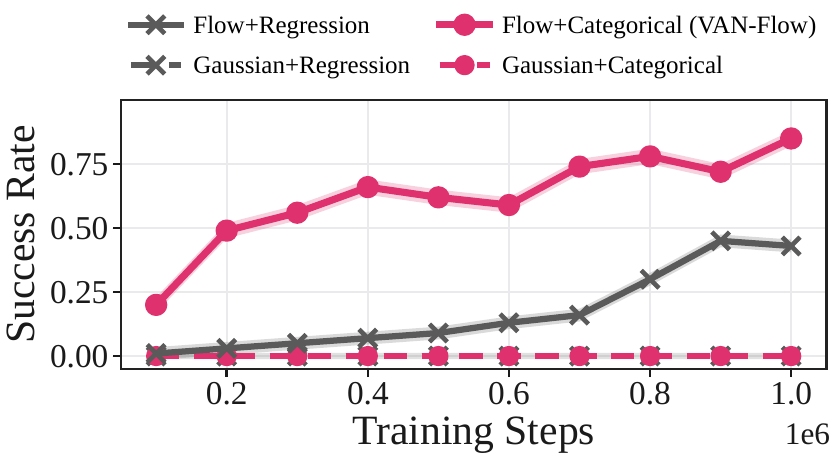}
    \caption{Effect of flow-based policies and categorical critics on \texttt{humanoidmaze-giant-navigate}.}
    \label{fig:policydits}
    \vspace{-\intextsep}
\end{wrapfigure}
Figure~\ref{fig:policydits} demonstrates the effect of combining a distributional critic with a flow-matching actor. In this result, our proposed combination (i.e., Flow+Categorical) consistently outperforms methods that rely partially on a regression-based critic or a Gaussian-based actor. Moreover, the consistent superiority of the flow-based actor over the Gaussian-based counterpart highlights its effectiveness in addressing long-horizon and sparse-reward problems. While Flow+Regression also exhibits improved performance, it is consistently outperformed by the \textbf{\color{magenta} VAN-Flow}, underscoring the importance of capturing and penalizing return variance via a categorical critic. We further validate the choice of flow matching over diffusion-based parameterizations in Appendix~\ref{app:flow_vs_diffusion}, where flow matching achieves over 90\% success rate with only 3 flow steps while a diffusion-based variant requires 20 steps for comparable performance.
\begin{wraptable}{r}{1.3\linewidth}
\centering
\vspace{-1.5em}
\caption{Ablation comparing return aggregation rules.}
\label{tab:ablation_vae_tasks}
\setlength{\tabcolsep}{4pt}
\renewcommand{\arraystretch}{0.95}
\resizebox{\linewidth}{!}{%
\begin{tabular}{ll ccccc}
\toprule
\textbf{Task}
& \textbf{Metric}
& $\mathbb{E}[Z]$
& \textbf{CVaR}
& \textbf{Entropic}
& \textbf{Mean--Var}
& {\color{magenta}$\mathcal{E}(Z)$} \\
\midrule
\multirow{2}{*}{\makecell[l]{\texttt{humanoidmaze}\\\texttt{-large-navigate}}}
& $\mathrm{Var}[Z]\!\downarrow$
& $1.00$ & $0.95$ & $0.97$ & $0.94$
& $\mathbf{0.93}$ \\
& $\mathrm{Success}\!\uparrow$
& $77$ & $83$ & $79$ & $80$
& $\mathbf{87}$ \\
\midrule
\multirow{2}{*}{\makecell[l]{\texttt{antmaze}\\\texttt{-teleport-navigate}}}
& $\mathrm{Var}[Z]\!\downarrow$
& $1.00$ & $1.08$ & $0.93$ & $0.90$
& $\mathbf{0.62}$ \\
& $\mathrm{Success}\!\uparrow$
& $38$ & $54$ & $44$ & $46$
& $\mathbf{71}$ \\
\midrule
\multirow{2}{*}{\texttt{antmaze-large-explore}}
& $\mathrm{Var}[Z]\!\downarrow$
& $1.00$ & $0.87$ & $0.72$ & $0.77$
& $\mathbf{0.66}$ \\
& $\mathrm{Success}\!\uparrow$
& $84$ & $2$ & $86$ & $85$
& $\mathbf{93}$ \\
\bottomrule
\end{tabular}}%
\vspace{-1em}
\end{wraptable}
\paragraph{Variance-averse Expectation.}
We evaluate whether the proposed \textbf{\color{magenta}variance-averse expectation} suppresses high-variance return distributions and how it compares to existing risk-sensitive aggregation rules. To this end, we replace the aggregation rule in~\eqref{eq:rejection_sampling} and~\eqref{eq:pro_actor_loss} with $\mathbb{E}[Z]$, CVaR, entropic risk, mean--variance, and {\color{magenta}$\mathcal{E}(Z)$}, while keeping all other components unchanged.\footnote{For each baseline, the associated hyperparameter (e.g., the confidence level of CVaR or the penalty coefficient $\beta$ of entropic risk and mean--variance) is selected via a per-task sweep that reports the best-performing value.}
We report two complementary metrics: $\mathrm{Var}[Z]$, which measures the dispersion of the learned return distribution on training samples after convergence and is normalized by its value under $\mathbb{E}[Z]$ within each task, and $\mathrm{Success}$, which measures task performance.

As shown in Table~\ref{tab:ablation_vae_tasks}, {\color{magenta}$\mathcal{E}(Z)$} consistently achieves the lowest $\mathrm{Var}[Z]$ across all tasks, and this variance reduction directly translates into the highest success rates. These results indicate that suppressing dispersion in return distributions is critical for reliable action selection in sparse, long-horizon offline RL. Existing risk-sensitive rules fall short for structural reasons. CVaR focuses on the lower tail of the distribution, while mean--variance and entropic risk introduce auxiliary penalty terms of the form $Q-\beta\mathrm{Var}$ or $Q-\beta\mathrm{Entropy}$ that require task-specific tuning. In contrast, $\mathcal{E}(Z)$ acts as a single aggregation operator over the categorical return distribution, redistributing probability mass to favor actions that are both high-return and low-variance without hard truncation or auxiliary penalties. This simple formulation consistently improves both variance and success using a fixed $\delta=2$ across nearly all tasks. Importantly, the operator is directly compatible with categorical distributional critics,
requiring no modification to the underlying distributional RL training procedure.
Additional comparisons with other risk measures (CPW, Norm, Wang, Sharpe, Sortino) are provided in Appendix~\ref{app:risk}.

\paragraph{$Q$ Guidance.}
\begin{wrapfigure}{r}{0.5\linewidth}
    \centering
    \includegraphics[width=\linewidth]{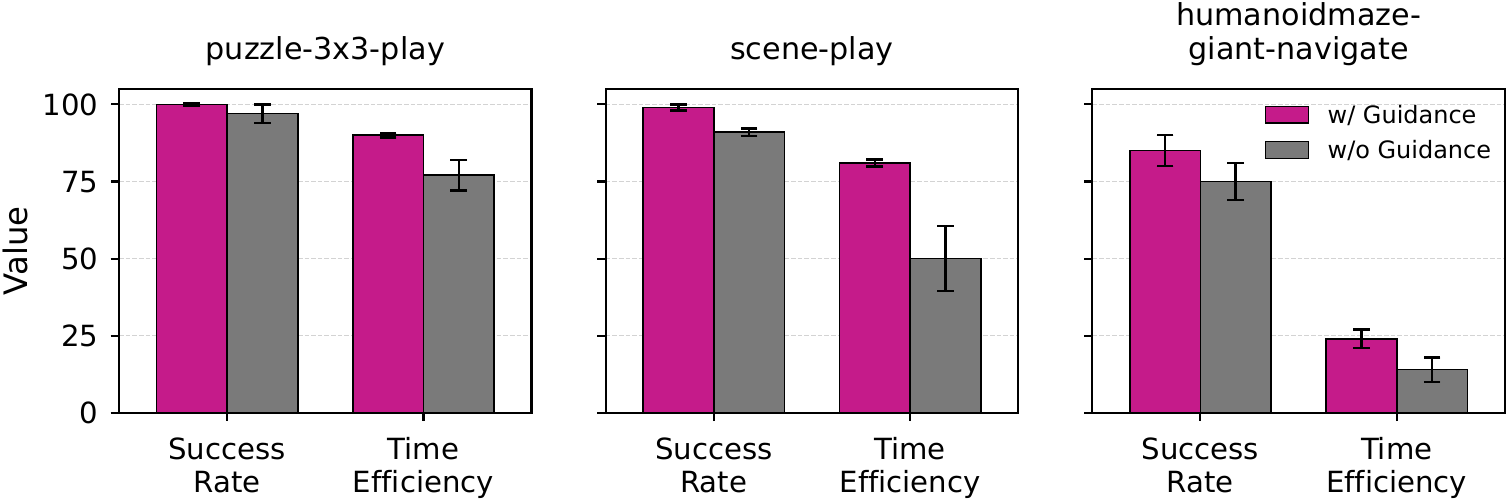}
    \caption{Effect of $Q$ guidance on policy performance across multiple tasks. Bars report mean success rate and time efficiency.}
    \label{fig:qguidance}
    \vspace{-\intextsep}
\end{wrapfigure}
We also examine the effect of the $Q$ guidance term 
in~\eqref{eq:pro_actor_loss} on actor performance. To evaluate its impact, we 
consider two metrics: (1) the success rate, which measures whether the actor 
successfully solves the task, and (2) the time efficiency, which assesses how 
rapidly the task is completed. The time efficiency is computed from the actual episode length $T$ required to reach the goal, normalized as $\mathrm{TE} = 100 \cdot (1 - T/T_{\max})$, where $T_{\max}$ denotes the maximum episode horizon. Higher values thus indicate faster task completion. Figure~\ref{fig:qguidance} clearly demonstrates that the $Q$-guided actor achieves both high success rates and high time efficiency, consistently outperforming an actor trained without $Q$ guidance. In contrast, while the actor without $Q$ guidance may achieve high success rates in some tasks, its time efficiency is often extremely low, indicating inefficient decision-making in the absence of explicit $Q$ guidance.

\subsection{A Close Look at {\color{magenta}VAN-Flow}}
\label{sec:close_look}
In this subsection, we provide a sensitivity analysis on the \texttt{humanoidmaze-giant-navigate} task to examine how key algorithmic parameters (e.g., the $n$-step horizon, the number of rejection samples, and the number of flow steps) affect the performance. The sensitivity analyses across various tasks are presented in Appendix~\ref{app:H}.

\begin{wrapfigure}{r}{0.45\linewidth}
    \centering
    \vspace{-2.0em}
    \includegraphics[width=\linewidth]{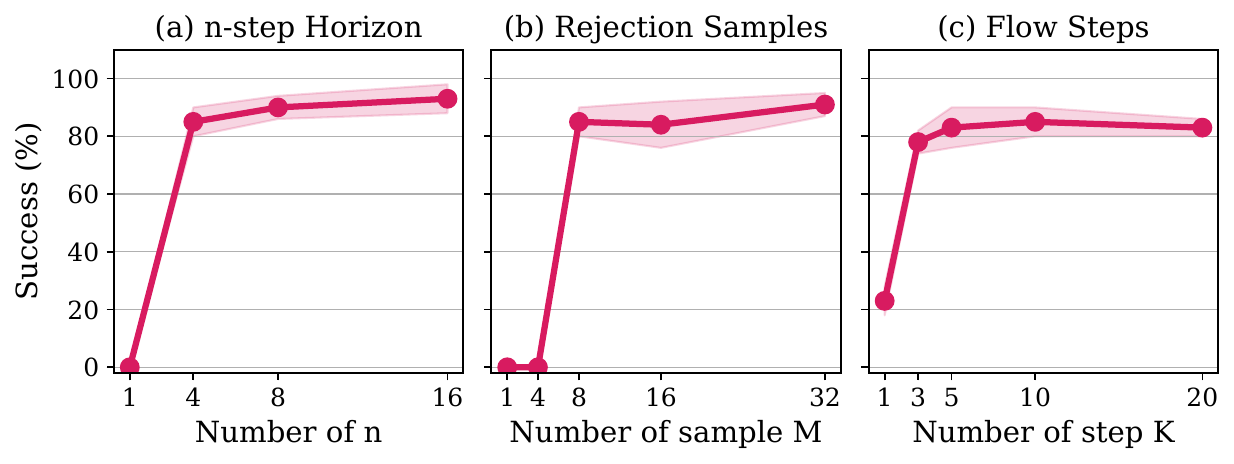}
    \caption{Sensitivity analysis with respect to key algorithmic parameters.}
    \label{fig:sensitive}
    \vspace{-0.5em}
\end{wrapfigure}

\paragraph{$n$-step Horizon.}
We examine the impact of varying the $n$-step return horizon with Figure~\ref{fig:sensitive}(a). For long-horizon tasks, increasing $n$ substantially improves performance by enabling more effective long-term credit assignment. However, excessively large values of $n$ lead to diminishing or even negative returns, reflecting the amplified variance and instability of long-horizon targets. Importantly, the proposed method remains robust across a wide range of $n$, indicating that variance-aware learning alleviates the sensitivity typically associated with long-horizon returns.

\paragraph{Number of Rejection Samples.}
The effect of the number of candidate actions used in rejection sampling is presented in Figure~\ref{fig:sensitive}(b). Performance consistently improves as more samples are considered, particularly on challenging tasks where conservative target action selection is critical. Beyond a moderate number of samples, the improvement gradually saturates, with most tasks exhibiting stable performance using fewer than $32$ candidates. This suggests that reliable variance-averse target action selection can be achieved without excessive computational overhead.

\paragraph{Number of Flow Steps.}
Figure~\ref{fig:sensitive}(c) analyzes sensitivity to the number of flow integration steps for action generation. A single integration step is not enough, but performance is already high with three to five steps and changes only mildly beyond that, so the proposed actor does not depend on precise numerical integration. Overall, these results show that evaluating critic values on flow-induced actions enables stable optimization with flexible integration granularity.

\subsection{Computational Complexity Analysis}
\begin{wraptable}{r}{0.5\linewidth}
    \vspace{-3.5em}
    \centering
    \footnotesize
    \captionsetup{justification=raggedright,singlelinecheck=false}
    \caption{Runtime comparison on \texttt{humanoidmaze-large-navigate}.}
    \label{tab:complexity}
    \renewcommand{\arraystretch}{1.1}
    \resizebox{\linewidth}{!}{
    \begin{tabular}{lccc}
    \toprule
    \textbf{Method} & \textbf{Complexity} & \textbf{Step (ms)} & \textbf{Wall (h)} \\
    \midrule
    FQL      & $\mathcal{O}(UKd^2)$    & 1.1 & 0.32 \\
    BFN      & $\mathcal{O}(UMKd^2)$   & 2.1 & 0.60 \\
    QC       & $\mathcal{O}(UMKd^2)$   & 3.0 & 0.83 \\
    VAN-Flow & $\mathcal{O}(UMKd^2)$   & 1.9 & 0.54 \\
    \bottomrule
    \end{tabular}
    }
    \vspace{-1.5em}
\end{wraptable}
VAN-Flow introduces additional computation due to rejection sampling and flow-based action generation. 
Table~\ref{tab:complexity} compares VAN-Flow with representative generative-policy methods, including FQL, BFN, and QC, in terms of both theoretical complexity and practical runtime. 
The overall training complexity of VAN-Flow is $\mathcal{O}(UMKd^2)$, where $U$ denotes the number of gradient updates, $d$ the network width (assuming fully-connected layers), $M$ the number of rejection-sampling candidates, and $K$ the number of flow steps. The factors $M$ and $K$ arise from evaluating $M$ candidate actions generated through $K$ flow transformations during policy optimization. 
Although this matches the asymptotic order of QC, the empirical overhead remains modest: VAN-Flow is faster than QC in both step time and wall-clock time, while remaining comparable to other generative-policy baselines. As discussed in Section~\ref{sec:close_look}, strong performance is achieved with relatively small values of $M$ and $K$, which keeps the practical computational cost low.

\section{Related Works}
\paragraph{Offline RL with $n$-step Returns.}
$n$-step returns have recently regained attention in offline RL as an effective tool for mitigating long-horizon bootstrapping bias. PiLaR~\cite{daley2024averaging} analyzes unfavorable bias--variance trade-offs of naive $n$-step returns and introduces compound return estimators. Decoupled Q-chunking~\cite{dqc2026} eliminates pessimistic bias by decoupling policy and critic horizons, while LEQ~\cite{park2025leq} stabilizes long-horizon learning using $n$-step $\lambda$-returns with expectile regression. Other routes to long-horizon offline RL augment transitions by temporal distance~\cite{lee2025temporaldistanceawaretransitionaugmentation} or build foresight into the agent~\cite{NEURIPS2024_1588dc2b,10421808}.

\paragraph{Generative Policies for Reinforcement Learning.}
Recent offline RL methods increasingly adopt expressive generative models to parameterize policies. Diffusion-based approaches, including Diffusion Policy~\cite{chi2023diffusion} and Diffuser~\cite{janner2022diffuser}, model complex action or trajectory distributions, while Diffusion-QL~\cite{wang2023diffusion} and IDQL~\cite{hansen2023dql} integrate critic-based guidance to move beyond pure imitation.
Flow-matching-based policies offer a scalable alternative, as exemplified by Flow Q-Learning~\cite{park2025fql}, Q-chunking~\cite{li2025qchunking}, and energy-weighted flow matching~\cite{zhang2025energyflow}.

\paragraph{Distributional and Risk-sensitive RL.}
Distributional RL methods such as C51~\cite{bellemare2017distributional}, 
QR-DQN~\cite{dabney2018quantile}, and IQN~\cite{dabney2018iqn} model the 
full return distribution but typically aggregate it via the standard 
expectation. Risk-sensitive variants based on CVaR~\cite{dabney2018iqn}, 
entropic risk, or mean--variance objectives modify the aggregation rule, 
but were designed to encode risk preferences in online RL rather than to 
suppress unreliable, high-variance targets in offline $n$-step learning. 
Distortion-based aggregation in actuarial risk theory~\cite{yaari1987,wang1996} 
reweights continuous distributions through the survival probability. 
Our variance-averse expectation $\mathcal{E}(\cdot)$ instead operates on 
discrete categorical critics with explicit normalization, reweighting all 
atoms without hard truncation while penalizing dispersion.

\section{Conclusion}
We investigate offline RL in sparse, long-horizon environments and show that naive horizon reduction via $n$-step returns fundamentally breaks down on high-variance datasets. To address this issue, we propose \textbf{\color{magenta}VAN-Flow}, a variance-averse offline RL framework that combines $n$-step learning with categorical distributional value estimation. VAN-Flow introduces a variance-averse expectation that downweights unreliable high-variance return targets and couples it with a flow-based actor guided by variance-averse $Q$ estimates and rejection sampling. Across challenging benchmarks, our method consistently outperforms strong baselines. To our knowledge, VAN-Flow is the first framework to select reliable behaviors for generative actors through variance-averse aggregation over categorical return distributions, and our results show that this choice matters for $n$-step offline RL under heterogeneous data.

\section*{Acknowledgments}
This work was partly supported by the Institute of Information \& Communications Technology Planning \& Evaluation (IITP) grant (No.~RS-2026-25519475), the National Research Foundation of Korea (NRF) grants (RS-2023-00278812, RS-2025-02214082, RS-2026-25522801) funded by the Korean government (MSIT), and Korea Institute of Police Technology (KIPoT) funded by the Korean National Police Agency \& Korea Customs Service (RS-2026-25536784). The computing resources were partly supported by Institute of Information \& communications Technology Planning \& Evaluation (IITP, AI Computing Support Project for R\&D) grant funded by the Korea government(MSIT) (RS-2026-25505492). The authors would like to thank Chanin Eom for his helpful discussions and comments during this research. The authors declare no competing interests.

\bibliography{ref}

\begin{thebibliography}{10}

\bibitem{wang2023diffusion}
Zhendong Wang, Jonathan~J. Hunt, and Mingyuan Zhou.
\newblock Diffusion policies as an expressive policy class for offline
  reinforcement learning.
\newblock In {\em International Conference on Learning Representations (ICLR)},
  2023.

\bibitem{chi2023diffusion}
Cheng Chi, Zhenjia Xu, Siyuan Feng, Eric Cousineau, Yilun Du, Benjamin
  Burchfiel, Russ Tedrake, and Shuran Song.
\newblock Diffusion policy: Visuomotor policy learning via action diffusion.
\newblock {\em The International Journal of Robotics Research},
  44(10--11):1684--1704, 2025.

\bibitem{park2025fql}
Seohong Park, Qiyang Li, and Sergey Levine.
\newblock Flow {Q}-learning.
\newblock In {\em International Conference on Machine Learning (ICML)}, 2025.

\bibitem{li2025qchunking}
Qiyang Li, Zhiyuan Zhou, and Sergey Levine.
\newblock Reinforcement learning with action chunking.
\newblock In {\em Advances in Neural Information Processing Systems (NeurIPS)},
  2025.

\bibitem{10343100}
Jeyeon Eo, Dongsu Lee, and Minhae Kwon.
\newblock The impact of dataset on offline reinforcement learning performance
  in uav-based emergency network recovery tasks.
\newblock {\em IEEE Communications Letters}, 28(5):1058--1061, 2024.

\bibitem{10610308}
Dongsu Lee, Chanin Eom, and Minhae Kwon.
\newblock {AD4RL}: {A}utonomous driving benchmarks for offline reinforcement
  learning with value-based dataset.
\newblock In {\em IEEE International Conference on Robotics and Automation},
  2024.

\bibitem{eom2024selective}
Chanin Eom, Dongsu Lee, and Minhae Kwon.
\newblock Selective imitation for efficient online reinforcement learning with
  pre-collected data.
\newblock {\em ICT Express}, 10(6):1308--1314, 2024.

\bibitem{bellemare2017distributional}
Marc~G. Bellemare, Will Dabney, and R{\'e}mi Munos.
\newblock A distributional perspective on reinforcement learning.
\newblock In {\em International Conference on Machine Learning (ICML)}, 2017.

\bibitem{dabney2018quantile}
Will Dabney, Mark Rowland, Marc~G. Bellemare, and R{\'e}mi Munos.
\newblock Distributional reinforcement learning with quantile regression.
\newblock In {\em AAAI Conference on Artificial Intelligence (AAAI)}, 2018.

\bibitem{dabney2018iqn}
Will Dabney, Georg Ostrovski, David Silver, and R{\'e}mi Munos.
\newblock Implicit quantile networks for distributional reinforcement learning.
\newblock In {\em International Conference on Machine Learning (ICML)}, 2018.

\bibitem{kearns2000bias}
Michael~J. Kearns and Satinder Singh.
\newblock Bias-variance error bounds for temporal difference updates.
\newblock In {\em Conference on Computational Learning Theory (COLT)}, 2000.

\bibitem{park2025horizon}
Seohong Park, Kevin Frans, Deepinder Mann, Benjamin Eysenbach, Aviral Kumar,
  and Sergey Levine.
\newblock Horizon reduction makes {RL} scalable.
\newblock In {\em Advances in Neural Information Processing Systems (NeurIPS)},
  2025.

\bibitem{precup2000eligibility}
Doina Precup, Richard~S. Sutton, and Satinder Singh.
\newblock Eligibility traces for off-policy policy evaluation.
\newblock In {\em International Conference on Machine Learning (ICML)}, 2000.

\bibitem{sutton2018rl}
Richard~S. Sutton and Andrew~G. Barto.
\newblock {\em Reinforcement Learning: An Introduction}.
\newblock MIT Press, Cambridge, MA, 2nd edition, 2018.

\bibitem{ren2021nearly}
Tongzheng Ren, Jialian Li, Bo~Dai, Simon~S. Du, and Sujay Sanghavi.
\newblock Nearly horizon-free offline reinforcement learning.
\newblock In {\em Advances in Neural Information Processing Systems (NeurIPS)},
  2021.

\bibitem{lee2023survey}
Dongsu Lee, Chanin Eom, Sungwoo Choi, Sungkwan Kim, and Minhae Kwon.
\newblock Survey on practical reinforcement learning : from imitation learning
  to offline reinforcement learning.
\newblock {\em Journal of Korean Institute of Communications and Information
  Sciences}, 48(11):1405--1417, 2023.

\bibitem{park2025leq}
Kwanyoung Park and Youngwoon Lee.
\newblock Model-based offline reinforcement learning with lower expectile
  {Q}-learning.
\newblock In {\em International Conference on Learning Representations (ICLR)},
  2025.

\bibitem{burns2023multifrequency}
Kaylee Burns, Tianhe Yu, Chelsea Finn, and Karol Hausman.
\newblock Offline reinforcement learning at multiple frequencies.
\newblock In {\em Conference on Robot Learning (CoRL)}, 2022.

\bibitem{chen2024deep}
Jiayu Chen, Bhargav Ganguly, Yang Xu, Yongsheng Mei, Tian Lan, and Vaneet
  Aggarwal.
\newblock Deep generative models for offline policy learning: Tutorial, survey,
  and perspectives on future directions.
\newblock {\em Transactions on Machine Learning Research}, 2024.

\bibitem{zhang2025energyflow}
Shiyuan Zhang, Weitong Zhang, and Quanquan Gu.
\newblock Energy-weighted flow matching for offline reinforcement learning.
\newblock In {\em International Conference on Learning Representations (ICLR)},
  2025.

\bibitem{kang2023diffusion}
Bingyi Kang, Xiao Ma, Chao Du, Tianyu Pang, and Shuicheng Yan.
\newblock Efficient diffusion policies for offline reinforcement learning.
\newblock In {\em Advances in Neural Information Processing Systems (NeurIPS)},
  2023.

\bibitem{chen2023highfidelity}
Huayu Chen, Cheng Lu, Chengyang Ying, Hang Su, and Jun Zhu.
\newblock Offline reinforcement learning via high-fidelity generative behavior
  modeling.
\newblock In {\em International Conference on Learning Representations (ICLR)},
  2023.

\bibitem{robert2004montecarlo}
Christian~P. Robert and George Casella.
\newblock {\em Monte Carlo Statistical Methods}.
\newblock Springer, New York, 2nd edition, 2004.

\bibitem{10423845}
Suzan~Ece Ada, Erhan Oztop, and Emre Ugur.
\newblock Diffusion policies for out-of-distribution generalization in offline
  reinforcement learning.
\newblock {\em IEEE Robotics and Automation Letters}, 9(4):3116--3123, 2024.

\bibitem{timothy2025ddpg}
Timothy~P. Lillicrap, Jonathan~J. Hunt, Alexander Pritzel, Nicolas Heess, Tom
  Erez, Yuval Tassa, David Silver, and Daan Wierstra.
\newblock Continuous control with deep reinforcement learning.
\newblock In {\em International Conference on Learning Representations (ICLR)},
  2016.

\bibitem{fujimoto2019off}
Scott Fujimoto, David Meger, and Doina Precup.
\newblock Off-policy deep reinforcement learning without exploration.
\newblock In {\em International Conference on Machine Learning (ICML)}, 2019.

\bibitem{NEURIPS2020_c16a53d5}
Aviral Kumar, Aurick Zhou, George Tucker, and Sergey Levine.
\newblock Conservative {Q}-learning for offline reinforcement learning.
\newblock In {\em Advances in Neural Information Processing Systems (NeurIPS)},
  2020.

\bibitem{fujimoto2018addressing}
Scott Fujimoto, Herke van Hoof, and David Meger.
\newblock Addressing function approximation error in actor-critic methods.
\newblock In {\em International Conference on Machine Learning (ICML)}, 2018.

\bibitem{vanhasselt2016double}
Hado van Hasselt, Arthur Guez, and David Silver.
\newblock Deep reinforcement learning with double {Q}-learning.
\newblock In {\em AAAI Conference on Artificial Intelligence (AAAI)}, 2016.

\bibitem{haarnoja2018soft}
Tuomas Haarnoja, Aurick Zhou, Pieter Abbeel, and Sergey Levine.
\newblock Soft actor-critic: Off-policy maximum entropy deep reinforcement
  learning with a stochastic actor.
\newblock In {\em International Conference on Machine Learning (ICML)}, 2018.

\bibitem{fu2020d4rl}
Justin Fu, Aviral Kumar, Ofir Nachum, George Tucker, and Sergey Levine.
\newblock {D4RL}: Datasets for deep data-driven reinforcement learning.
\newblock arXiv preprint arXiv:2004.07219, 2020.

\bibitem{tarasov2023corl}
Denis Tarasov, Alexander Nikulin, Dmitry Akimov, Vladislav Kurenkov, and Sergey
  Kolesnikov.
\newblock {CORL}: Research-oriented deep offline reinforcement learning
  library.
\newblock In {\em Advances in Neural Information Processing Systems (NeurIPS)
  Datasets and Benchmarks Track}, 2023.

\bibitem{rafailov2024d5rl}
Rafael Rafailov, Kyle~Beltran Hatch, Anikait Singh, Aviral Kumar, Laura Smith,
  Ilya Kostrikov, Philippe Hansen-Estruch, Victor Kolev, Philip~J. Ball, Jiajun
  Wu, Sergey Levine, and Chelsea Finn.
\newblock {D5RL}: Diverse datasets for data-driven deep reinforcement learning.
\newblock {\em Reinforcement Learning Journal}, 5:2178--2197, 2024.

\bibitem{park2025ogbench}
Seohong Park, Kevin Frans, Benjamin Eysenbach, and Sergey Levine.
\newblock {OGBench}: Benchmarking offline goal-conditioned {RL}.
\newblock In {\em International Conference on Learning Representations (ICLR)},
  2025.

\bibitem{kostrikov2022iql}
Ilya Kostrikov, Ashvin Nair, and Sergey Levine.
\newblock Offline reinforcement learning with implicit {Q}-learning.
\newblock In {\em International Conference on Learning Representations (ICLR)},
  2022.

\bibitem{tarasov2023revisiting}
Denis Tarasov, Vladislav Kurenkov, Alexander Nikulin, and Sergey Kolesnikov.
\newblock Revisiting the minimalist approach to offline reinforcement learning.
\newblock In {\em Advances in Neural Information Processing Systems (NeurIPS)},
  2023.

\bibitem{park2023hiql}
Seohong Park, Dibya Ghosh, Benjamin Eysenbach, and Sergey Levine.
\newblock {HIQL}: Offline goal-conditioned {RL} with latent states as actions.
\newblock In {\em Advances in Neural Information Processing Systems (NeurIPS)},
  2023.

\bibitem{munos2016safe}
R{\'e}mi Munos, Tom Stepleton, Anna Harutyunyan, and Marc~G. Bellemare.
\newblock Safe and efficient off-policy reinforcement learning.
\newblock In {\em Advances in Neural Information Processing Systems (NeurIPS)},
  2016.

\bibitem{kozuno2021revisiting}
Tadashi Kozuno, Yunhao Tang, Mark Rowland, R{\'e}mi Munos, Steven Kapturowski,
  Will Dabney, Michal Valko, and David Abel.
\newblock Revisiting {P}eng's {Q}($\lambda$) for modern reinforcement learning.
\newblock In {\em International Conference on Machine Learning (ICML)}, 2021.

\bibitem{peng2024deadly}
Zhiyong Peng, Yadong Liu, and Zongtan Zhou.
\newblock Deadly triad matters for offline reinforcement learning.
\newblock {\em Knowledge-Based Systems}, 284:111341, 2024.

\bibitem{mark2025policyagnostic}
Max~Sobol Mark, Tian Gao, Georgia~Gabriela Sampaio, Mohan~Kumar Srirama, Archit
  Sharma, Chelsea Finn, and Aviral Kumar.
\newblock Policy agnostic {RL}: Offline {RL} and online {RL} fine-tuning of any
  class and backbone.
\newblock arXiv preprint arXiv:2412.06685, 2024.

\bibitem{maron2018distributional}
Gabriel Barth-Maron, Matthew~W. Hoffman, David Budden, Will Dabney, Dan Horgan,
  Dhruva TB, Alistair Muldal, Nicolas Heess, and Timothy Lillicrap.
\newblock Distributed distributional deterministic policy gradients.
\newblock In {\em International Conference on Learning Representations (ICLR)},
  2018.

\bibitem{fujimoto2021minimalist}
Scott Fujimoto and Shixiang~Shane Gu.
\newblock A minimalist approach to offline reinforcement learning.
\newblock In {\em Advances in Neural Information Processing Systems (NeurIPS)},
  2021.

\bibitem{daley2024averaging}
Brett Daley, Martha White, and Marlos~C. Machado.
\newblock Averaging $n$-step returns reduces variance in reinforcement
  learning.
\newblock In {\em International Conference on Machine Learning (ICML)}, 2024.

\bibitem{dqc2026}
Qiyang Li, Seohong Park, and Sergey Levine.
\newblock Decoupled {Q}-chunking.
\newblock In {\em International Conference on Learning Representations (ICLR)},
  2026.

\bibitem{lee2025temporaldistanceawaretransitionaugmentation}
Dongsu Lee and Minhae Kwon.
\newblock Temporal distance-aware transition augmentation for offline
  model-based reinforcement learning.
\newblock In {\em International Conference on Machine Learning}, 2025.

\bibitem{NEURIPS2024_1588dc2b}
Dongsu Lee and Minhae Kwon.
\newblock Episodic future thinking mechanism for multi-agent reinforcement
  learning.
\newblock In {\em Advances in Neural Information Processing Systems (NeurIPS)},
  2024.

\bibitem{10421808}
Dongsu Lee and Minhae Kwon.
\newblock Foresighted decisions for inter-vehicle interactions: An offline
  reinforcement learning approach.
\newblock In {\em IEEE International Conference on Intelligent Transportation
  Systems}, 2023.

\bibitem{janner2022diffuser}
Michael Janner, Yilun Du, Joshua~B. Tenenbaum, and Sergey Levine.
\newblock Planning with diffusion for flexible behavior synthesis.
\newblock In {\em International Conference on Machine Learning (ICML)}, 2022.

\bibitem{hansen2023dql}
Philippe Hansen-Estruch, Ilya Kostrikov, Michael Janner, Jakub~Grudzien Kuba,
  and Sergey Levine.
\newblock {IDQL}: Implicit {Q}-learning as an actor-critic method with
  diffusion policies.
\newblock arXiv preprint arXiv:2304.10573, 2023.

\bibitem{yaari1987}
Menahem~E. Yaari.
\newblock The dual theory of choice under risk.
\newblock {\em Econometrica}, 55(1):95--115, 1987.

\bibitem{wang1996}
Shaun Wang.
\newblock Premium calculation by transforming the layer premium density.
\newblock {\em ASTIN Bulletin}, 26(1):71--92, 1996.

\bibitem{shaked2007stochastic}
Moshe Shaked and J.~George Shanthikumar.
\newblock {\em Stochastic Orders}.
\newblock Springer, New York, 2007.

\bibitem{kingma2015adam}
Diederik~P. Kingma and Jimmy Ba.
\newblock {Adam}: A method for stochastic optimization.
\newblock In {\em International Conference on Learning Representations (ICLR)},
  2015.

\bibitem{hendrycks2016gelu}
Dan Hendrycks and Kevin Gimpel.
\newblock {Gaussian} error linear units ({GELUs}).
\newblock arXiv preprint arXiv:1606.08415, 2016.

\bibitem{kangcikm}
Guhyeon Kang, Jaehwi Lee, and Minhae Kwon.
\newblock Simple actors and deep critics for scalable reinforcement learning.
\newblock In {\em ACM International Conference on Information and Knowledge
  Management}, 2026.

\bibitem{chen2024score}
Huayu Chen, Cheng Lu, Zhengyi Wang, Hang Su, and Jun Zhu.
\newblock Score regularized policy optimization through diffusion behavior.
\newblock In {\em International Conference on Learning Representations (ICLR)},
  2024.

\bibitem{ding2024consistency}
Zihan Ding and Chi Jin.
\newblock Consistency models as a rich and efficient policy class for
  reinforcement learning.
\newblock In {\em International Conference on Learning Representations (ICLR)},
  2024.

\bibitem{nair2020awac}
Ashvin Nair, Abhishek Gupta, Murtaza Dalal, and Sergey Levine.
\newblock {AWAC}: Accelerating online reinforcement learning with offline
  datasets.
\newblock arXiv preprint arXiv:2006.09359, 2020.

\bibitem{11080103}
Dongsu Lee and Minhae Kwon.
\newblock Scenario-free autonomous driving with multi-task offline-to-online
  reinforcement learning.
\newblock {\em IEEE Transactions on Intelligent Transportation Systems},
  26(9):13317--13330, 2025.

\bibitem{10969773}
Dongsu Lee and Minhae Kwon.
\newblock Episodic future thinking with offline reinforcement learning for
  autonomous driving.
\newblock {\em IEEE Internet of Things Journal}, 12(11):17012--17023, 2025.

\bibitem{kang2025offlineramp}
Guhyeon Kang, Jaehwi Lee, Chanin Eom, and Minhae Kwon.
\newblock Offline reinforcement learning based autonomous driving on complex
  roads with on-ramp merging and bottleneck.
\newblock {\em Transactions of the Korean Society of Automotive Engineers},
  33(6):465--479, 2025.

\bibitem{sangeunksae}
Sangeun Park, Chanin Eom, and Minhae Kwon.
\newblock x{LSTM}-based noise-robust driving characteristic inference network
  for {V2I} systems.
\newblock {\em Transactions of the Korean Society of Automotive Engineers},
  33(7):487--502, 2025.

\bibitem{eo2024offline}
Jeyeon Eo, Dongsu Lee, and Minhae Kwon.
\newblock Offline reinforcement learning based uav training for disaster
  network recovery.
\newblock {\em Journal of Korean Institute of Communications and Information
  Sciences}, 49(1):1--11, 2024.

\bibitem{lee2026ai}
Jaehwi Lee, Chanin Eom, Chan Kim, Kyeongsoo Kim, Hyeongdo Lee, Hyunsu Kang, and
  Minhae Kwon.
\newblock Ai fire support officer: Military decision support system based on
  reward adaptive reinforcement learning.
\newblock {\em Journal of Korean Institute of Communications and Information
  Sciences}, 51(1):209--222, 2026.

\bibitem{lee2025deep}
Jaehwi Lee, Chanin Eom, Kyeongsoo Kim, Hyunsu Kang, and Minhae Kwon.
\newblock Deep reinforcement learning based weapon-target assignment to support
  military decision-making.
\newblock {\em Journal of Korean Institute of Communications and Information
  Sciences}, 50(6):884--895, 2025.

\end{thebibliography}
\bibliographystyle{unsrt}

\clearpage
\appendix
\setcounter{section}{0}

\newcommand{\appline}[3]{%
  \noindent\textbf{#1}\hspace{0.6em}#2\dotfill\pageref{#3}\par
}
\newcommand{\appsubline}[3]{%
  \noindent\hspace*{1.7em}\textbf{#1}\hspace{0.6em}#2\dotfill\pageref{#3}\par
}
\newcommand{\appsubsubline}[3]{%
  \noindent\hspace*{3.2em}\textbf{#1}\hspace{0.6em}#2\dotfill\pageref{#3}\par
}

\thispagestyle{empty}
\vspace*{2em}
\begin{center}
  {\LARGE\bfseries Appendix}
\end{center}
\vspace{1.2em}

\appline{A}{Notation Tables}{app:A}

\vspace{0.6em}
\appline{B}{Theoretical Analysis of the Variance-Averse Expectation}{app:B}
\appsubline{B.1}{Proof of Theorem~\ref{thm:VaE_mps}(i)}{app:b1}
\appsubline{B.2}{Design Considerations}{app:b2}
\appsubline{B.3}{Proof of Theorem~\ref{thm:VaE_mps}(ii): Discretization Error}{app:b3}

\vspace{0.6em}
\appline{C}{Projection Operator}{app:C}

\vspace{0.6em}
\appline{D}{Benchmarks and Task Descriptions}{app:D}
\appsubline{D.1}{D4RL}{app:d1}
\appsubline{D.2}{OGBench}{app:d2}

\vspace{0.6em}
\appline{E}{Implementation Details of VAN-Flow}{app:E}

\vspace{0.6em}
\appline{F}{Component-wise Ablation Study}{app:component_ablation}

\vspace{0.6em}
\appline{G}{Empirical Comparison of Flow Matching and Diffusion Policies}{app:flow_vs_diffusion}

\vspace{0.6em}
\appline{H}{Baselines Details}{app:F}
\appsubline{H.1}{Baselines Description}{app:f1}
\appsubline{H.2}{Sources of Reported Results}{app:f2}

\vspace{0.6em}
\appline{I}{Extensive Results of Performance Comparisons}{app:G}
\appsubline{I.1}{Learning Curve of VAN-Flow}{app:g1}
\appsubline{I.2}{Full Performance Table}{app:g2}
\appsubline{I.3}{Performance Comparison in Denser Reward Environments}{app:g3}
\appsubline{I.4}{D4RL Locomotion Tasks}{app:g4}

\vspace{0.6em}
\appline{J}{Off-to-Online Fine-Tuning Results}{app:offon}

\vspace{0.6em}
\appline{K}{Additional Comparisons with Risk-Sensitive Objectives}{app:risk}

\vspace{0.6em}
\appline{L}{Full Results on Sensitivity Analyses}{app:H}
\appsubline{L.1}{$n$-step Horizon}{app:h1}
\appsubline{L.2}{Number of Rejection Samples}{app:h2}
\appsubline{L.3}{Number of Flow Steps}{app:h3}
\appsubline{L.4}{Aversion Coefficient $\delta$}{app:h4}
\appsubline{L.5}{Categorical Support}{app:h5}
\appsubline{L.6}{Number of Atoms}{app:h6}

\vspace{0.6em}
\appline{M}{Broader Impact and Limitation}{app:limit}

\clearpage

\section{Notation Tables}\label{app:A}

Table~\ref{tab:notation} summarizes the key notations used throughout the paper.

\begin{table}[h]
\caption{Notation table.}
\centering
\small
\renewcommand{\arraystretch}{1.15}
\setlength{\tabcolsep}{6pt}
\begin{tabular}{ll}
\toprule
\textbf{Notation} & \textbf{Description} \\
\midrule
$\mathcal{M}=(\mathcal{S},\mathcal{A},P,r,\gamma)$ & Markov decision process \\
$\pi_\theta(\cdot\mid s)$ & Flow-based policy (actor) with parameters $\theta$ \\
$v^\pi_\theta(\tau,s,x)$ & Flow-matching velocity field (policy parameterization) \\
$Q_\psi(s,a)$ & Scalar critic (used in standard 1-step formulations) \\
$Z_\psi(s,a)$ & Categorical return distribution critic used in VAN-Flow \\
$Z(s,a)$ & Return random variable (categorical distribution) for $(s,a)$ \\
$\{z_i\}_{i=1}^{I}$ & Fixed support atoms for categorical distribution \\
$\{p_i(s,a)\}_{i=1}^{I}$ & Atom probabilities of $Z_\psi(s,a)$, $\sum_i p_i=1$ \\
$\mathcal{D}$ & Offline dataset of transitions/segments \\
$G_t^{(n)}=\sum_{k=0}^{n-1}\gamma^k r_{t+k}$ & $n$-step truncated return \\
$\mathcal{T}^{(n)}$ & $n$-step Bellman operator (scalar) \\
$\mathcal{T}^{(n)}_z$ & $n$-step distributional Bellman operator \\
$\Phi(\cdot)$ & Categorical projection operator onto fixed atoms \\
$\mathcal{C}(z_i)=\mathbb{P}(Z\le z_i)$ & CDF of the return random variable $Z$ at $z_i$ \\
$\bar{\mathcal{C}}(z_i)=1-\mathcal{C}(z_i)$ & Survival function of $Z$ at $z_i$, with $\bar{\mathcal{C}}(z_0)=1$ \\
$\Omega_Z(u)$ & Generalized quantile function of return distribution $Z$ \\
$F_Z(u)=\int_0^u \Omega_Z(t)dt$ & Integrated quantile function \\
$\mathcal{E}(Z)$ & Variance-averse expectation, implemented operator (Def.~\ref{def:VaE}) \\
$\mathcal{E}^{\mathrm{sp}}(Z)$ & Spectral form of $\mathcal{E}$, exact finite-sum representation~\eqref{eq:VaE_spectral} \\
$\varepsilon_\delta(Z)$ & Discretization-gap certificate of Theorem~\ref{thm:VaE_mps}(ii) \\
$\delta$ & Aversion intensity for $\mathcal{E}(\cdot)$ \\
$\boldsymbol{a}_t^*$ & Action selected via variance-averse rejection sampling \\
$\mathbf{A}_t^\pi$ & Set of $M$ candidate actions sampled from the policy \\
$\tau\sim\mathrm{Unif}[0,1]$ & Flow time variable \\
$K$ & Number of flow discretization steps \\
$M$ & Number of rejection-sampling candidates \\
$\lambda$ & Weight for flow-matching regression term \\
$\le_{\mathrm{cx}}$ & Convex order \\
\bottomrule
\end{tabular}
\label{tab:notation}
\end{table}

\section{Theoretical Analysis of the Variance-Averse Expectation}\label{app:B}

\subsection{Proof of Theorem~\ref{thm:VaE_mps}(i)}\label{app:b1}

Throughout Appendix~\ref{app:B}, $Z$ is a categorical return distribution with probabilities $p_i\ge0$, $\sum_{i=1}^{I}p_i=1$, on ordered atoms $z_1<\cdots<z_I$. We write $\mathcal{C}_i:=\mathcal{C}(z_i)=\sum_{j\le i}p_j$ for its CDF, with $\mathcal{C}_0:=0$, and $S_i:=\bar{\mathcal{C}}(z_i)=1-\mathcal{C}_i$ for its survival function, so that $S_0=1$, $S_I=0$, and $p_i=S_{i-1}-S_i$. The generalized quantile function of $Z$ is
\begin{align}
\Omega_Z(u)=z_i,\qquad u\in\big(\mathcal{C}_{i-1},\,\mathcal{C}_i\big],
\label{eq:quantile_app}
\end{align}
and $w_i^{\star}:=S_{i-1}^{\delta+1}-S_i^{\delta+1}$ denotes the weights of the spectral form~\eqref{eq:VaE_spectral}. We first verify that the two expressions in~\eqref{eq:VaE_spectral} coincide.

\begin{lemma}[Exact finite-sum representation]\label{lem:finite_sum}
For every $\delta\ge0$,
$(\delta+1)\int_0^1\Omega_Z(u)(1-u)^{\delta}\,du=\sum_{i=1}^{I}w_i^{\star}z_i$.
The weights $w_i^{\star}$ are nonnegative and sum to one, and $w_i^{\star}=p_i$ when $\delta=0$.
\end{lemma}
\begin{proof}
By~\eqref{eq:quantile_app}, $\Omega_Z$ is constant on each interval $(\mathcal{C}_{i-1},\mathcal{C}_i]$, hence
\begin{align}
(\delta+1)\int_0^1\Omega_Z(u)(1-u)^{\delta}\,du
&=\sum_{i=1}^{I}z_i\int_{\mathcal{C}_{i-1}}^{\mathcal{C}_i}(\delta+1)(1-u)^{\delta}\,du\notag\\
&=\sum_{i=1}^{I}z_i\Big[(1-\mathcal{C}_{i-1})^{\delta+1}-(1-\mathcal{C}_i)^{\delta+1}\Big]
=\sum_{i=1}^{I}w_i^{\star}z_i.
\end{align}
Nonnegativity follows from $S_{i-1}\ge S_i$, the sum telescopes to $S_0^{\delta+1}-S_I^{\delta+1}=1$, and for $\delta=0$ we have $w_i^{\star}=S_{i-1}-S_i=p_i$.
\end{proof}

\begin{proof}[Proof of Theorem~\ref{thm:VaE_mps}(i)]
By Lemma~\ref{lem:finite_sum}, $\mathcal{E}^{\mathrm{sp}}(Z)=(\delta+1)\int_0^1\Omega_Z(u)(1-u)^{\delta}\,du$.
Define $F_Z(u):=\int_0^u\Omega_Z(t)\,dt$. Integrating by parts,
\begin{align}
\mathcal{E}^{\mathrm{sp}}(Z)
&=(\delta+1)\Big[F_Z(u)(1-u)^{\delta}\Big]_{0}^{1}
+(\delta+1)\delta\int_0^1F_Z(u)(1-u)^{\delta-1}\,du\notag\\
&=(\delta+1)\delta\int_0^1F_Z(u)(1-u)^{\delta-1}\,du,
\label{eq:vae_double_int_app}
\end{align}
where the boundary term vanishes because $F_Z(0)=0$, $F_Z$ is bounded, and $(1-u)^{\delta}\to0$ as $u\to1$ for $\delta>0$. For $0<\delta<1$ the weight $(1-u)^{\delta-1}$ is integrable on $[0,1]$, so the right-hand side is well defined.

Now let $Z^{\dagger}$ and $Z^{\circ}$ satisfy $\mathbb{E}[Z^{\dagger}]=\mathbb{E}[Z^{\circ}]$ and $Z^{\dagger}\le_{\mathrm{cx}}Z^{\circ}$. The quantile characterization of the convex order (see, e.g.,~\cite[Sec.~3.A]{shaked2007stochastic}) states that, for random variables with equal means, $Z^{\dagger}\le_{\mathrm{cx}}Z^{\circ}$ holds if and only if
\begin{align}
F_{Z^{\dagger}}(u)=\int_0^u\Omega_{Z^{\dagger}}(t)\,dt
\;\ge\;
\int_0^u\Omega_{Z^{\circ}}(t)\,dt=F_{Z^{\circ}}(u)
\qquad\text{for all }u\in[0,1].
\label{eq:cx_order_quantile_app}
\end{align}
Substituting~\eqref{eq:cx_order_quantile_app} into~\eqref{eq:vae_double_int_app},
\begin{align}
\mathcal{E}^{\mathrm{sp}}(Z^{\dagger})-\mathcal{E}^{\mathrm{sp}}(Z^{\circ})
=(\delta+1)\delta\int_0^1
\Big[F_{Z^{\dagger}}(u)-F_{Z^{\circ}}(u)\Big](1-u)^{\delta-1}\,du
\;\ge\;0,
\end{align}
since the bracket is nonnegative for all $u\in[0,1]$ and $(1-u)^{\delta-1}>0$ on $(0,1)$. Hence $\mathcal{E}^{\mathrm{sp}}(Z^{\circ})\le\mathcal{E}^{\mathrm{sp}}(Z^{\dagger})$.
\end{proof}

The operator is variance averse, but it never prefers a uniformly worse return distribution. The following proposition makes this precise.

\begin{proposition}[Monotonicity under first-order stochastic dominance]\label{prop:fosd}
Let $Z^{1}$ and $Z^{2}$ be categorical return distributions on the same atoms with $\mathcal{C}_{Z^{1}}(z)\le\mathcal{C}_{Z^{2}}(z)$ for every $z$, that is, $Z^{1}$ first-order stochastically dominates $Z^{2}$. Then $\mathcal{E}^{\mathrm{sp}}(Z^{1})\ge\mathcal{E}^{\mathrm{sp}}(Z^{2})$ for every $\delta\ge0$.
\end{proposition}
\begin{proof}
First-order stochastic dominance is equivalent to $\Omega_{Z^{1}}(u)\ge\Omega_{Z^{2}}(u)$ for all $u\in(0,1]$. By Lemma~\ref{lem:finite_sum},
$\mathcal{E}^{\mathrm{sp}}(Z^{1})-\mathcal{E}^{\mathrm{sp}}(Z^{2})=(\delta+1)\int_0^1\big[\Omega_{Z^{1}}(u)-\Omega_{Z^{2}}(u)\big](1-u)^{\delta}\,du\ge0$,
since both the bracket and the weight $(1-u)^{\delta}$ are nonnegative.
\end{proof}

Consequently, an action whose returns are uniformly higher than those of another action is never rejected in favor of it, whatever its variance. Variance affects the ranking only among actions that are not ordered by first-order stochastic dominance, which is exactly the situation addressed by Theorem~\ref{thm:VaE_mps}. By Proposition~\ref{prop:gap}, the implemented operator $\mathcal{E}$ inherits this property up to the discretization gap $R\,\varepsilon_\delta(Z)$.

\subsection{Design Considerations}\label{app:b2}

This subsection elaborates on two design considerations of the variance-averse 
expectation $\mathcal{E}(Z)$ defined in Equation~\eqref{eq:VaE}. Both are quantified by the discretization analysis of Appendix~\ref{app:b3}.
\paragraph{Design Choice on the Boundary Atom.}
By construction $\mathcal{C}(z_I)=1$, and thus the reweighting in~\eqref{eq:VaE} 
assigns zero weight to the largest support atom $z_I$. More generally, any atom above which the critic places no mass receives zero weight, whereas the spectral form~\eqref{eq:VaE_spectral} assigns such an atom the weight $\bar{\mathcal{C}}(z_{i-1})^{\delta+1}$. Proposition~\ref{prop:gap} bounds the effect of this truncation (and of the normalization) on the value by $R\,\varepsilon_\delta(Z)$, a quantity that is small whenever the probability mass is spread over several atoms. In the standard C51-style parameterization the support 
$[V_{\min}, V_{\max}]$ is deliberately chosen to bracket the plausible return 
range, and the projection operator $\Phi(\cdot)$ distributes learned 
probability mass across multiple interior atoms, leaving the boundary atom 
$z_I$ with negligible mass in practice. We measure the resulting gap on trained critics in Appendix~\ref{app:b3}, where $\mathcal{E}$ and $\mathcal{E}^{\mathrm{sp}}$ select the same rejection-sampling candidate in $96.9\%$ of states.

\paragraph{Numerical Stability.}
The denominator of~\eqref{eq:VaE}, given by 
$\sum_{j=1}^{I} p_j(1-\mathcal{C}(z_j))^\delta=(1-\varepsilon_\delta(Z))/(\delta+1)$, is strictly positive whenever some atom with positive probability has positive survival probability. The categorical critic outputs are parameterized via softmax, 
ensuring strictly positive atom probabilities, and the C51 projection further 
spreads mass across multiple atoms. As a safeguard, we additionally include a 
small numerical stabilizer $\eta = 10^{-8}$ in the denominator. The denominator becomes small only when the mass concentrates on a single atom ($\max_i p_i\approx1$), which is also the regime in which the bound of Theorem~\ref{thm:VaE_mps}(ii) is loose (Remarks~\ref{rem:loose} and~\ref{rem:eps}). Throughout 
all reported experiments, we observed no numerical instability arising from 
this term. For critics that output nearly deterministic returns, the spectral form $\mathcal{E}^{\mathrm{sp}}$ is a drop-in alternative that requires no normalization and returns $z_i$ exactly for a point mass at $z_i$.

\subsection{Proof of Theorem~\ref{thm:VaE_mps}(ii): Discretization Error of the Implemented Operator}\label{app:b3}

We use the notation of Appendix~\ref{app:b1} and additionally write $a_i:=p_iS_i^{\delta}$ for the unnormalized weights of Definition~\ref{def:VaE}, $w_i:=a_i/\sum_{j}a_j$ for its normalized weights (we take the stabilizer $\eta=0$ here and treat $\eta>0$ in Remark~\ref{rem:eps}), $\varepsilon_\delta:=1-(\delta+1)\sum_i a_i$ as in~\eqref{eq:VaE_gap}, $\rho:=\max_i p_i$, and $R:=z_I-z_1$. In this notation $\mathcal{E}(Z)=\sum_i w_iz_i$ and $\mathcal{E}^{\mathrm{sp}}(Z)=\sum_i w_i^{\star}z_i$. Substituting $x=1-u$ in Lemma~\ref{lem:finite_sum},
\begin{align}
w_i^{\star}=\int_{S_i}^{S_{i-1}}(\delta+1)\,x^{\delta}\,dx,
\label{eq:wstar_integral}
\end{align}
so that $(\delta+1)a_i=(\delta+1)p_iS_i^{\delta}$ is the right-endpoint rule (in $u$) applied to the same integral, and $\mathcal{E}$ is this rule followed by normalization, which justifies the description in Theorem~\ref{thm:VaE_mps}(ii).

\begin{proposition}[Discretization error]\label{prop:gap}
For every $\delta>0$ the following hold.
\begin{enumerate}[label=(\alph*),leftmargin=*,nosep]
\item $(\delta+1)a_i\le w_i^{\star}\le(\delta+1)p_iS_{i-1}^{\delta}$ for all $i$. Consequently, $\varepsilon_\delta\in[0,1]$.
\item $\sum_{i=1}^{I}|w_i^{\star}-w_i|\le2\,\varepsilon_\delta$.
\item $|\mathcal{E}^{\mathrm{sp}}(Z)-\mathcal{E}(Z)|\le R\,\varepsilon_\delta$.
\item $\varepsilon_\delta\le(\delta+1)\rho$ for every $\delta>0$, and $\varepsilon_\delta\le\delta(\delta+1)\sum_ip_i^{2}$ for $\delta\ge1$.
\end{enumerate}
\end{proposition}
\begin{proof}
(a) The integrand in~\eqref{eq:wstar_integral} is nondecreasing on $[S_i,S_{i-1}]$, an interval of length $p_i$, so it lies between $(\delta+1)S_i^{\delta}$ and $(\delta+1)S_{i-1}^{\delta}$, and integrating gives (a). Summing the lower bound over $i$ and using Lemma~\ref{lem:finite_sum} yields $(\delta+1)\sum_ia_i\le\sum_iw_i^{\star}=1$, i.e., $\varepsilon_\delta\in[0,1]$.

(b) Write $w_i^{\star}-w_i=\big(w_i^{\star}-(\delta+1)a_i\big)+(\delta+1)a_i\big(1-\tfrac{1}{1-\varepsilon_\delta}\big)$. The first term is nonnegative by (a) and sums to $\varepsilon_\delta$, whereas the second is nonpositive and sums to $-\varepsilon_\delta$. Hence $\sum_i|w_i^{\star}-w_i|\le\varepsilon_\delta+\varepsilon_\delta$.

(c) Both weight vectors sum to one, so for any constant $c$, $\mathcal{E}^{\mathrm{sp}}(Z)-\mathcal{E}(Z)=\sum_i(w_i^{\star}-w_i)(z_i-c)$. Taking $c=(z_1+z_I)/2$ gives $|z_i-c|\le R/2$, and (b) yields the claim.

(d) By (a) and~\eqref{eq:wstar_integral}, $\varepsilon_\delta=\sum_i\int_{S_i}^{S_{i-1}}(\delta+1)\big(x^{\delta}-S_i^{\delta}\big)\,dx$. For the first bound, monotonicity of $x\mapsto x^{\delta}$ gives $x^{\delta}-S_i^{\delta}\le S_{i-1}^{\delta}-S_i^{\delta}$ on $[S_i,S_{i-1}]$, so
\begin{align}
\varepsilon_\delta\le(\delta+1)\sum_i p_i\big(S_{i-1}^{\delta}-S_i^{\delta}\big)
&\le(\delta+1)\,\rho\sum_i\big(S_{i-1}^{\delta}-S_i^{\delta}\big)\notag\\
&=(\delta+1)\,\rho\,\big(S_0^{\delta}-S_I^{\delta}\big)=(\delta+1)\,\rho .
\end{align}
For the second bound, if $\delta\ge1$ the mean value theorem gives $x^{\delta}-S_i^{\delta}\le\delta x^{\delta-1}(x-S_i)\le\delta p_i$ on $[S_i,S_{i-1}]$ (using $x\le1$), so $\varepsilon_\delta\le(\delta+1)\delta\sum_ip_i^{2}$.
\end{proof}

Proposition~\ref{prop:gap}(c)--(d) is the bound stated in Theorem~\ref{thm:VaE_mps}(ii). The remaining claims of part (ii) are the two corollaries below.

\begin{corollary}[Ordering transfer]\label{cor:ordering}
Let $\mathbb{E}[Z^{\dagger}]=\mathbb{E}[Z^{\circ}]$ and $Z^{\dagger}\le_{\mathrm{cx}}Z^{\circ}$. Then
\begin{align}
\mathcal{E}(Z^{\circ})\le\mathcal{E}(Z^{\dagger})+R\big[\varepsilon_\delta(Z^{\dagger})+\varepsilon_\delta(Z^{\circ})\big],
\end{align}
which is~\eqref{eq:VaE_transfer}. In particular, if the convex-order margin satisfies $\mathcal{E}^{\mathrm{sp}}(Z^{\dagger})-\mathcal{E}^{\mathrm{sp}}(Z^{\circ})>R[\varepsilon_\delta(Z^{\dagger})+\varepsilon_\delta(Z^{\circ})]$, then $\mathcal{E}(Z^{\dagger})>\mathcal{E}(Z^{\circ})$. That is, the implemented operator induces the same ordering as the spectral form whenever the margin exceeds the certified error budget.
\end{corollary}
\begin{proof}
By Theorem~\ref{thm:VaE_mps}(i), $\mathcal{E}^{\mathrm{sp}}(Z^{\circ})\le\mathcal{E}^{\mathrm{sp}}(Z^{\dagger})$. Applying Proposition~\ref{prop:gap}(c) to both distributions,
$\mathcal{E}(Z^{\circ})\le\mathcal{E}^{\mathrm{sp}}(Z^{\circ})+R\varepsilon_\delta(Z^{\circ})\le\mathcal{E}^{\mathrm{sp}}(Z^{\dagger})+R\varepsilon_\delta(Z^{\circ})\le\mathcal{E}(Z^{\dagger})+R[\varepsilon_\delta(Z^{\dagger})+\varepsilon_\delta(Z^{\circ})]$. The second claim follows from the triangle inequality in the same way.
\end{proof}

\begin{corollary}[Vanishing gap under the C51 projection]\label{cor:c51}
Let $Y$ be a target return supported on $[V_{\min},V_{\max}]$ whose CDF is $L$-Lipschitz, and let $Z=\Phi(Y)$ be its projection~\eqref{eq:c51_projection} onto $I$ atoms with spacing $\Delta=(V_{\max}-V_{\min})/(I-1)$. Then $\rho\le2L\Delta=O(1/I)$ and
\begin{align}
|\mathcal{E}^{\mathrm{sp}}(Z)-\mathcal{E}(Z)|\le R\,\varepsilon_\delta(Z)\le 2(\delta+1)LR\,\Delta=O(1/I).
\end{align}
\end{corollary}
\begin{proof}
The projection moves the mass at $y$ only to the two atoms adjacent to $y$, so atom $z_i$ receives mass only from $\{y:|y-z_i|<\Delta\}$ and $p_i\le F_Y(z_i+\Delta)-F_Y(z_i-\Delta)\le2L\Delta$ (no clipping occurs because $Y$ is supported on $[V_{\min},V_{\max}]$). Apply Proposition~\ref{prop:gap}(c)--(d).
\end{proof}

\begin{remark}[When the bound is loose]\label{rem:loose}
The certificate $R\,\varepsilon_\delta(Z)$ is small whenever the critic spreads its mass over several atoms ($\rho\ll1$) and is vacuous for a nearly deterministic return ($\rho\approx1$): for a point mass at $z_i$ the spectral form returns exactly $z_i$ ($w_i^{\star}=1$), whereas Definition~\ref{def:VaE} assigns that atom the weight $p_iS_i^{\delta}=0$, so that $\mathcal{E}$ is determined by the (vanishingly small) mass below $z_i$. The Lipschitz assumption of Corollary~\ref{cor:c51} excludes this regime. The guarantee therefore covers return distributions that are genuinely spread---the regime produced by $n$-step bootstrapped targets under the C51 projection with $I=101$ atoms, as verified below---and does not cover deterministic returns.
\end{remark}

\begin{remark}[Numerical stabilizer]\label{rem:eps}
With the stabilizer $\eta>0$ of Appendix~\ref{app:b2}, the implemented weights sum to $\sum_ja_j/(\sum_ja_j+\eta)<1$, which adds at most $\max_i|z_i|\cdot\eta/(\sum_ja_j+\eta)$ to the error in Proposition~\ref{prop:gap}(c). This term is negligible unless $\sum_ja_j=(1-\varepsilon_\delta)/(\delta+1)\lesssim\eta$, i.e., unless $\varepsilon_\delta\approx1$, which is exactly the regime of Remark~\ref{rem:loose}. In that regime $\mathcal{E}^{\mathrm{sp}}$, which needs no normalization, is the preferable drop-in replacement.
\end{remark}

\paragraph{Empirical size of the gap on trained critics.}
To confirm that the regime of Corollary~\ref{cor:c51}, rather than that of Remark~\ref{rem:loose}, is the one encountered in practice, we measured the gap on critics trained with the default configuration on \texttt{antmaze-teleport-navigate}, using 256 randomly selected dataset states and $M=32$ candidate actions per state. For every candidate action drawn from the flow actor we evaluated $\mathcal{E}$ and $\mathcal{E}^{\mathrm{sp}}$. Table~\ref{tab:gap_empirical} reports the results: the realized gap is small (median $0.37\%$ of the value range $R$), both operators select the same candidate in $96.9\%$ of states, and their candidate rankings are strongly correlated (Kendall's $\tau=0.99$), indicating that the discretization gap is negligible at the atom resolution used throughout the paper.

\begin{table}[h]
\centering
\footnotesize
\caption{Discretization gap between the implemented operator $\mathcal{E}$ and its spectral form $\mathcal{E}^{\mathrm{sp}}$ on trained critics.}
\label{tab:gap_empirical}
\renewcommand{\arraystretch}{1.1}
\begin{tabular}{lc}
\toprule
\textbf{Measurement} & \textbf{Value} \\
\midrule
Median realized gap $|\mathcal{E}-\mathcal{E}^{\mathrm{sp}}|$ (fraction of the value range $R$) & $0.37\%$ \\
States in which $\mathcal{E}$ and $\mathcal{E}^{\mathrm{sp}}$ select the same candidate & $96.9\%$ \\
Kendall's $\tau$ between the two candidate rankings & $0.99$ \\
\bottomrule
\end{tabular}
\end{table}

\section{Projection Operator}\label{app:C}

This appendix details the categorical projection operator $\Phi(\cdot)$ used in
\eqref{eq:DnRL_BO} and \eqref{eq:proposed_criticBB}, following the standard C51-style projection
introduced in~\cite{bellemare2017distributional}.

\paragraph{Fixed support.}
We fix $I$ atoms $\{z_i\}_{i=1}^{I}$ uniformly spanning an interval $[V_{\min},V_{\max}]$:
\begin{equation}
z_i = V_{\min} + (i-1)\Delta,\qquad \Delta=\frac{V_{\max}-V_{\min}}{I-1}.
\end{equation}
The categorical critic outputs a probability vector $p(s,a)\in\Delta^{I}$ over these atoms.

\paragraph{Target distribution and clamping.}
Let the (unprojected) target distribution be a random variable $Y$ obtained by shifting the next-state
return distribution by the $n$-step return:
\begin{equation}
Y \;=\; G_t^{(n)} + \gamma^n Z_\psi(s_{t+n},a_{t+n}^*).
\end{equation}
Because $Y$ may fall outside $[V_{\min},V_{\max}]$, we clamp support values using
\begin{equation}
\bar{y}=\mathrm{clip}(y, V_{\min}, V_{\max}).
\end{equation}

\paragraph{Projection $\Phi$: mass transfer to neighboring atoms.}
Suppose $Y$ is represented categorically as probabilities $\{q_j\}_{j=1}^{I}$ on some values $\{y_j\}_{j=1}^{I}$
(e.g., $y_j = G_t^{(n)} + \gamma^n z_j$ when the next distribution is supported on $\{z_j\}$).
The projection $\Phi$ maps this distribution onto $\{z_i\}$ by distributing each mass $q_j$
to the two nearest atoms:
\begin{align}
b_j &= \frac{\bar{y}_j - V_{\min}}{\Delta}+1, \\
\ell_j &= \lfloor b_j \rfloor,\qquad u_j=\lceil b_j \rceil,
\end{align}
and define the projected probabilities $m_i$ by
\begin{equation}
m_i \;=\;\sum_{j=1}^{I} q_j \left[
\mathbf{1}\{u_j=i\}(b_j-\ell_j) + \mathbf{1}\{\ell_j=i\}(u_j-b_j)
\right],
\label{eq:c51_projection}
\end{equation}
with the convention that when $\ell_j=u_j$, all mass is assigned to that atom.
Then $\Phi(Y)$ is the categorical distribution supported on $\{z_i\}$ with probabilities $\{m_i\}_{i=1}^{I}$.
\clearpage
\section{Benchmarks and Task Descriptions}\label{app:D}
This section summarizes the benchmark environments used in our experiments,
with emphasis on properties that are directly relevant to the problem studied in this paper:
\emph{sparse rewards}, \emph{long-horizon credit assignment},
and \emph{dataset-induced return variance}.
Rather than providing a comprehensive description of each benchmark,
we highlight how different task and dataset variants give rise to distinct variance regimes
for $n$-step return learning in offline RL.

\subsection{D4RL \texttt{antMaze}}\label{app:d1}
\begin{figure}[h]
    \centering
    \includegraphics[width=0.7\linewidth]{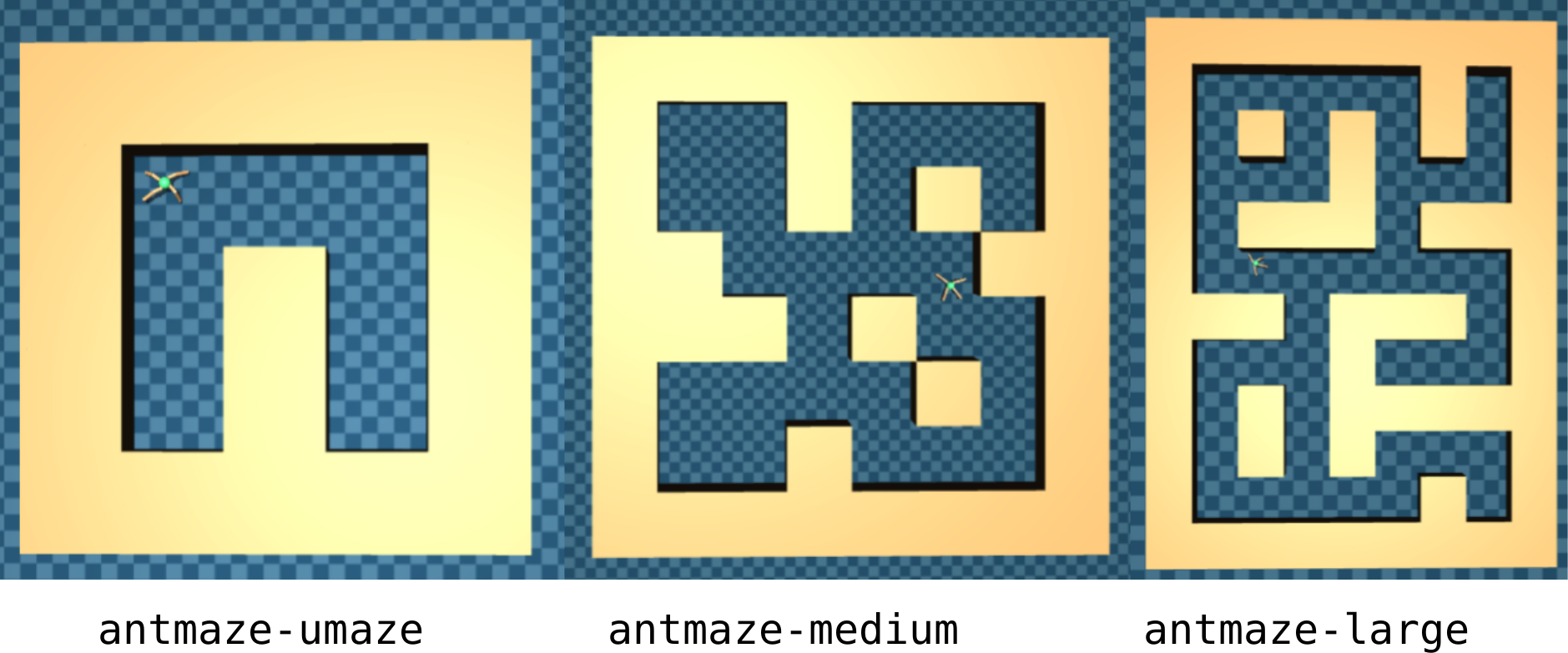}
    \caption{D4RL}
    \label{fig:D4RL}
\end{figure}
We evaluate on the \texttt{AntMaze} tasks from D4RL, a canonical benchmark
for offline RL under sparse rewards and long planning horizons.
AntMaze has become a standard testbed for studying long-horizon credit assignment
and robustness to dataset-induced distributional challenges.
In these environments, rewards are extremely sparse and are typically provided
only when the agent reaches the goal region, making effective long-horizon value propagation essential.

\paragraph{Horizon and variance characteristics.}
AntMaze provides a controlled setting in which both the effective planning horizon
and dataset-induced return variance can be systematically varied.

\begin{itemize}[leftmargin=14pt, nosep]
    \item \textbf{Maze scales:} \texttt{umaze}, \texttt{medium}, and \texttt{large}
    progressively increase the episode horizon and the depth of credit assignment.
    \item \textbf{Dataset variants:} \texttt{play} datasets are collected using relatively
    stationary behavior policies, resulting in lower return variance,
    whereas \texttt{diverse} datasets are generated from heterogeneous behavior policies
    and induce substantially higher return variance.
\end{itemize}

As the maze scale and dataset heterogeneity increase,
$n$-step return targets aggregate rewards over increasingly diverse trajectories,
thereby amplifying return variance.
Consequently, AntMaze serves as a particularly suitable benchmark for evaluating the stability of
long-horizon and multi-step learning methods in offline RL.

\subsection{OGBench}\label{app:d2}
\begin{figure}[h]
    \centering
    \includegraphics[width=0.6\linewidth]{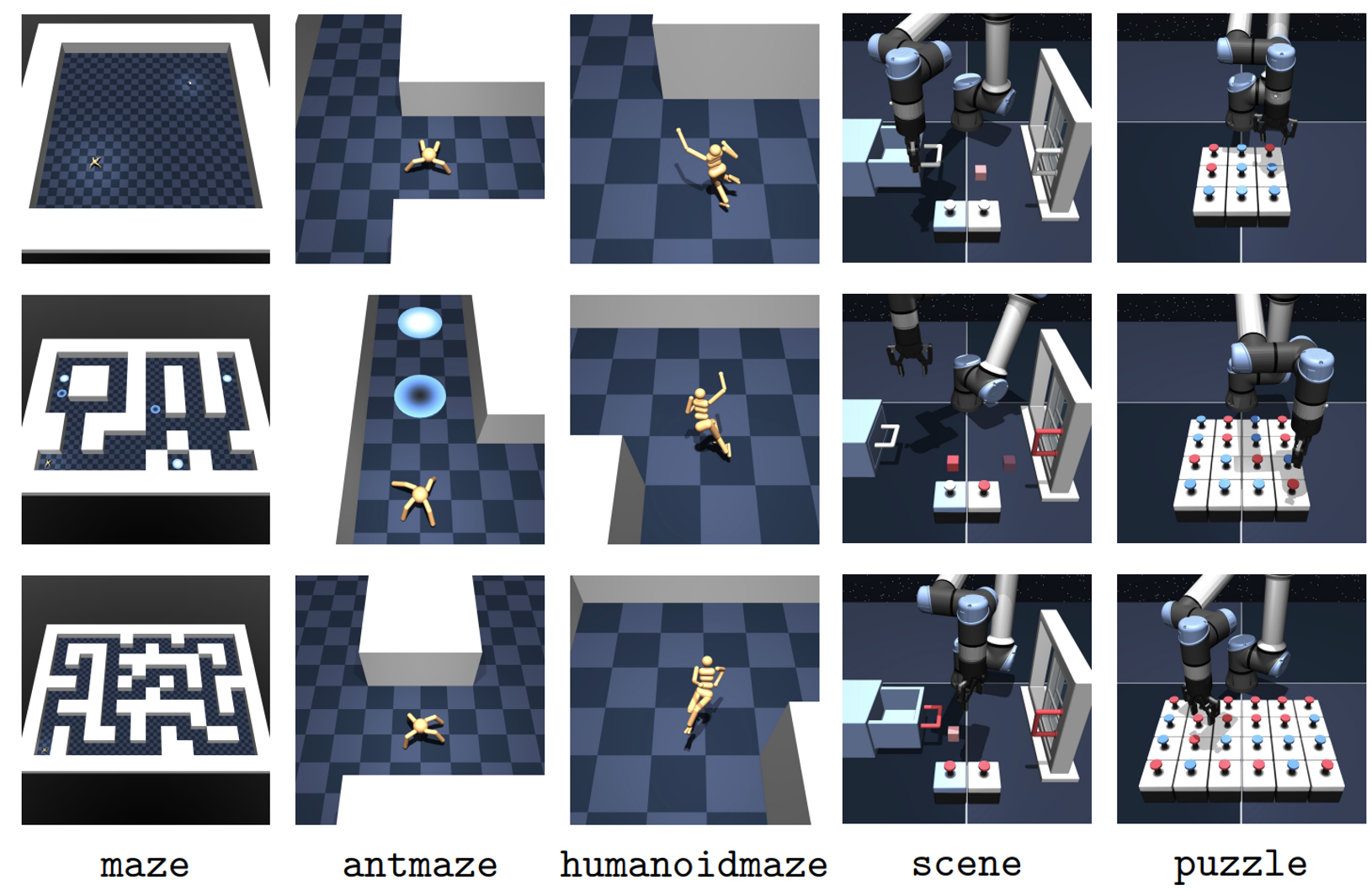}
    \caption{OGBench}
    \label{fig:OGBench}
\end{figure}
OGBench is a benchmark suite designed to evaluate offline RL in environments
characterized by long horizons, sparse or deceptive rewards, and limited dataset coverage.
In contrast to D4RL AntMaze, OGBench offers a broader collection of tasks
with diverse dynamics, action spaces, and dataset construction processes.

\paragraph{Relevance to return variance.}
A distinguishing feature of OGBench is that different dataset variants within the same environment
exhibit substantial differences in return dispersion.
This property enables controlled comparisons between low-variance and high-variance offline datasets
under identical task dynamics.

\paragraph{Task families used.}
We focus on state-based OGBench tasks that are particularly sensitive to
long-horizon value estimation and return variance:

\begin{itemize}[leftmargin=14pt, nosep]
    \item \textbf{AntMaze:}
    \texttt{antmaze-large-navigate}, \texttt{antmaze-teleport-navigate},
    \texttt{antmaze-large-explore}, and \texttt{antmaze-giant-navigate},
    which require long-horizon planning and are highly sensitive to dataset coverage.
    
    \item \textbf{HumanoidMaze:}
    \texttt{humanoidmaze-large-navigate},
    \texttt{humanoidmaze-giant-navigate},
    which combine extremely long horizons with high-dimensional action spaces.
    
    \item \textbf{Scene and Puzzle:}
    \texttt{scene-play},
    \texttt{puzzle-3x3-play}, and \texttt{puzzle-3x3-noisy},
    which exhibit strong multi-modality and heightened sensitivity to out-of-distribution actions.
\end{itemize}

\paragraph{Low-variance vs.\ high-variance datasets.}
Following the main text, we categorize dataset variants according to their induced return variance.

\begin{itemize}[leftmargin=14pt, nosep]
    \item \texttt{navigate} and \texttt{play} datasets are treated as relatively
    low-variance datasets, collected using more stationary behavior policies.
    \item \texttt{explore} and \texttt{noisy} datasets are treated as high-variance datasets,
    where non-stationary or noisy behavior policies induce large dispersion in return targets.
\end{itemize}

These high-variance datasets pose a significant challenge for naïve $n$-step return methods,
as increasing the return horizon directly amplifies target variance.
Consequently, OGBench serves as a particularly suitable testbed for evaluating
variance-aware learning strategies in offline RL.
\clearpage
\section{Implementation Details of VAN-Flow}\label{app:E}
We provide the hyperparameter details of the proposed method in Table~\ref{tab:vanflow_hparams}.
\begin{table}[h]
\caption{Hyperparameters for VAN-Flow.}
\label{tab:vanflow_hparams}
\begin{tabular}{>{\raggedright\arraybackslash}p{\dimexpr0.5\linewidth-2\tabcolsep\relax}
                >{\raggedright\arraybackslash}p{\dimexpr0.5\linewidth-2\tabcolsep\relax}}
\toprule
\textbf{Hyperparameter} & \textbf{Value} \\
\midrule
\multicolumn{2}{l}{\textit{Optimization}} \\
Learning rate (critic) $\alpha_Z$ & $3\times10^{-4}$ \\
Learning rate (actor) $\alpha_\pi$ & $3\times10^{-4}$ \\
Optimizer & Adam~\citep{kingma2015adam} \\
Gradient steps & $10^6$ (OGBench), $5 \times 10^5$ (D4RL) \\
Minibatch size & 256 \\
\midrule
\multicolumn{2}{l}{\textit{Network Architecture}} \\
MLP dimensions & [512, 512, 512, 512] \\
Nonlinearity & GELU~\citep{hendrycks2016gelu} \\
Target network smoothing coefficient $\rho$ & 0.005 \\
\midrule
\multicolumn{2}{l}{\textit{Discounting and Returns}} \\
Discount factor $\gamma$ & \makecell[l]{0.995 (default) \\ 0.999 (antmaze-giant, antmaze-teleport, \\ humanoidmaze)} \\
n-step return $n$ & \makecell[l]{
4 (default) \\
2 (antmaze-teleport) \\
3 (antmaze-large-explore, scene-play) \\
8 (antmaze-giant)} \\
\midrule
\multicolumn{2}{l}{\textit{Variance-Averse Expectation}} \\
Aversion intensity $\delta$ & \makecell[l]{2 (default) \\ 7 (antmaze-teleport)} \\
\midrule
\multicolumn{2}{l}{\textit{Distributional Critic}} \\
$V_{\min}$ & -200 (default), -1000 (antmaze-giant, antmaze-teleport, humanoidmaze) \\
$V_{\max}$ & 0 \\
Critic ensemble size & 2 \\
Critic aggregation method & mean \\
Categorical atoms & 101  \\
\midrule
\multicolumn{2}{l}{\textit{Flow-based Actor}} \\
Flow steps $K$ & 10  \\
Number of candidate actions $M$ & 8  \\
Flow-matching regression term weight $\lambda$ & \makecell[l]{3 (default) \\ 10 (antmaze-giant, humanoidmaze)\\ 300 (antmaze-teleport)} \\
\bottomrule
\end{tabular}
\end{table}

\paragraph{Reward Preprocessing.}
We rescale all rewards to $\{-1, 0\}$ prior to training, so that the discounted 
return remains within the categorical critic's support $[V_{\min}, 0]$. The 
lower bound is set as $V_{\min} \approx -1/(1-\gamma)$, which corresponds to 
the worst-case return when the goal is never reached. This transformation adds 
a state-independent constant to all $Q$-values and therefore does not affect 
the optimal policy.

\clearpage
\section{Component-wise Ablation Study}\label{app:component_ablation}

This appendix presents a component-wise ablation that isolates the individual and
joint contributions of the three core components of \textbf{\color{magenta}VAN-Flow}:
the variance-averse expectation (VE), the categorical distributional critic (Dist),
and the flow-based actor (Flow).

\paragraph{Setup.}
All variants are trained with $n$-step returns under identical hyperparameters,
differing only in which components are enabled. We evaluate across two
action-dimension regimes: low-dimensional locomotion (\texttt{antmaze-teleport}
and \texttt{antmaze-giant}, both 8-dim) and a high-dimensional task
(\texttt{humanoidmaze-large}, 21-dim).
\begin{table}[h]
\centering
\footnotesize
\caption{Component-wise ablation across action-dimension regimes. All variants use
$n$-step returns. Scores denote success rate (\%), averaged over five random seeds.}
\label{tab:component_ablation}
\renewcommand{\arraystretch}{1.1}
\setlength{\tabcolsep}{6pt}
\resizebox{0.95\linewidth}{!}{%
\begin{tabular}{lccc|ccc}
\toprule
\textbf{Variant} & \textbf{Dist} & \textbf{VE} & \textbf{Flow}
& \texttt{antmaze-teleport} (8d)
& \texttt{antmaze-giant} (8d)
& \texttt{humanoidmaze-large} (21d) \\
\midrule
\rowcolor{MagentaShade}
\textbf{VAN-Flow (Ours)} & \cmark & \cmark & \cmark
& $\mathbf{71\!\pm\!7}$ & $\mathbf{58\!\pm\!4}$ & $\mathbf{87\!\pm\!6}$ \\
\midrule
w/o VE          & \cmark & \xmark & \cmark & $38\!\pm\!7$  & $36\!\pm\!6$  & $77\!\pm\!4$ \\
w/o Flow        & \cmark & \cmark & \xmark & $42\!\pm\!7$  & $38\!\pm\!13$ & $0\!\pm\!0$ \\
w/o VE \& Flow  & \cmark & \xmark & \xmark & $36\!\pm\!9$  & $32\!\pm\!5$  & $0\!\pm\!0$ \\
\bottomrule
\end{tabular}}
\end{table}

\paragraph{Discussion.}
On low-dimensional tasks (\texttt{antmaze-teleport}, \texttt{antmaze-giant}), 
removing VE from VAN-Flow causes a substantial performance drop 
($71 \rightarrow 38$ on teleport; $58 \rightarrow 36$ on giant), 
while removing Flow leads to a smaller decrease 
($71 \rightarrow 42$; $58 \rightarrow 38$). 
This indicates that, under $n$-step returns, actor expressiveness alone 
provides little benefit unless return variance is properly controlled. 
VE downweights unreliable high-variance return targets, while Flow 
provides the expressiveness required to exploit the resulting value estimates.

The high-dimensional task \texttt{humanoidmaze-large} (21d) makes this 
synergy even clearer. Removing Flow causes complete failure 
($87 \rightarrow 0$), consistent with Figure~\ref{fig:policydits}, 
since a Gaussian actor cannot represent the multi-modal action distribution 
required by the task. Removing VE instead is far less catastrophic 
($87 \rightarrow 77$), but still leaves a meaningful gap. 
These results jointly demonstrate that VE and Flow play complementary roles: 
Flow supplies the action expressiveness, while VE ensures the value estimates 
guiding the actor are reliable. How much capacity the actor and the critic each need is studied separately in~\cite{kangcikm}.

\clearpage
\section{Empirical Comparison of Flow Matching and Diffusion Policies}\label{app:flow_vs_diffusion}

This appendix provides the empirical justification for adopting flow matching 
as the policy backbone of \textbf{\color{magenta}VAN-Flow}, rather than the more 
widely used diffusion-based parameterization (e.g., DDPM~\cite{wang2023diffusion}).

\paragraph{Motivation.}
Diffusion-based actors~\cite{wang2023diffusion,chi2023diffusion} 
typically require many denoising steps to produce high-quality actions, 
substantially increasing both training and inference cost. Flow matching 
instead parameterizes the policy via a deterministic ODE that enables 
high-quality samples with far fewer integration 
steps~\cite{park2025fql}, which is particularly attractive in 
actor--critic offline RL where the actor is repeatedly evaluated during 
both training and inference.
\paragraph{Experimental Setup.}
To validate this design choice within our framework, we instantiate a 
diffusion-based variant of our method, denoted \textbf{VAN-DDPM}, by replacing 
the flow-matching actor in \textbf{\color{magenta}VAN-Flow} with a DDPM-style 
denoising actor while keeping all other components (categorical critic, 
variance-averse expectation, rejection sampling, $n$-step return) unchanged.
We evaluate both variants on \texttt{humanoidmaze-large-navigate} over five 
random seeds, varying the number of denoising/flow steps. We report the 
success rate as performance and the cumulative inference time per action 
generation in milliseconds.

\begin{table}[h]
\centering
\footnotesize
\caption{Comparison between flow matching (VAN-Flow) and diffusion (VAN-DDPM) 
policies on \texttt{humanoidmaze-large-navigate}. Performance denotes success 
rate (\%); Time denotes cumulative inference time per action (ms). Bold entries 
indicate the first step count achieving $>$90\% success rate.}
\label{tab:flow_vs_diffusion}
\renewcommand{\arraystretch}{1.4}
\setlength{\tabcolsep}{14pt}
\begin{tabular}{c cc cc}
\toprule
& \multicolumn{2}{c}{\textbf{VAN-DDPM}}
& \multicolumn{2}{c}{\textbf{\color{magenta}VAN-Flow}} \\
\cmidrule(lr){2-3} \cmidrule(lr){4-5}
\textbf{Steps} & \textbf{Perf.} & \textbf{Time (ms)} & \textbf{Perf.} & \textbf{Time (ms)} \\
\midrule
1  & $0\pmv{0}$    & 2.0  & $0\pmv{0}$            & 1.6 \\
3  & $0\pmv{0}$    & 3.8  & $\mathbf{96\pmv{3}}$  & \textbf{3.3} \\
5  & $0\pmv{0}$    & 5.7  & $96\pmv{2}$           & 5.0 \\
10 & $57\pmv{17}$  & 8.4  & $87\pmv{6}$           & 7.2 \\
20 & $\mathbf{94\pmv{4}}$ & \textbf{11.9} & $76\pmv{6}$ & 10.2 \\
\bottomrule
\end{tabular}
\end{table}

\paragraph{Discussion.}
As shown in Table~\ref{tab:flow_vs_diffusion}, \textbf{\color{magenta}VAN-Flow} 
achieves over 90\% success rate with only $3$ flow steps (3.3ms), whereas 
\textbf{VAN-DDPM} requires $20$ denoising steps (11.9ms) to reach a comparable 
level, corresponding to roughly a $3.6\times$ reduction in inference cost. 
We further observe that VAN-DDPM exhibits a much steeper sensitivity to the 
number of denoising steps: it produces near-zero success rates for $\leq 5$ 
steps, indicating that diffusion-based actors require sufficiently many 
denoising iterations to generate in-distribution actions. In contrast, 
flow matching attains strong performance even at very low step counts. 
Overall, these results justify our use of flow matching as the policy 
backbone, balancing high sample quality and low computational cost.

\section{Baselines Details}\label{app:F}
\subsection{Baselines Description}\label{app:f1}
We briefly summarize the baselines used in our comparisons. All methods are run with publicly available
implementations when possible, or re-implemented following the original papers, using the same dataset splits
and evaluation protocols as \textbf{\color{magenta}VAN-Flow}.

\paragraph{Pure offline RL (1-step)}
\begin{itemize}[leftmargin=14pt, nosep]
\item \textbf{IQL~\cite{kostrikov2022iql}:} An offline actor–critic method that employs expectile regression for value learning and advantage-weighted updates for policy improvement.
\item \textbf{ReBRAC~\cite{tarasov2023revisiting}:} A behavior-regularized actor–critic method designed to improve robustness against out-of-distribution (OOD) actions.
\item \textbf{HIQL~\cite{park2023hiql}:} A hierarchical extension of IQL tailored for long-horizon decision-making problems.
\end{itemize}

\paragraph{Flow-based policies}
\begin{itemize}[leftmargin=14pt, nosep]
\item \textbf{FQL~\cite{park2025fql}:} Flow Q-Learning, an offline actor–critic method that employs a flow-matching actor guided by a learned critic.
\item \textbf{BFN / QC~\cite{li2025qchunking}:} Flow-based policy families with value-guided sampling. We adopt the hyperparameter settings recommended by the authors.
\item \textbf{FQL-$n$, BFN-$n$:} Variants of the original algorithms that incorporate an $n$-step temporal-difference (TD) target.
\end{itemize}

\paragraph{$n$-step / horizon reduction methods}
\begin{itemize}[leftmargin=14pt, nosep]
\item \textbf{LEQ~\cite{park2025leq}:} A long-horizon stabilization method that combines $n$-step $\lambda$-returns with expectile regression for improved value estimation stability.

\item \textbf{QC~\cite{li2025qchunking}:} An action-chunking framework based on flow-based policy families with value-guided sampling. We adopt the hyperparameter settings recommended by the authors.
\end{itemize}

\paragraph{Distributional critics}
\begin{itemize}[leftmargin=14pt, nosep]
\item \textbf{PA-RL~\cite{mark2025policyagnostic}:} A policy-agnostic actor–critic algorithm that leverages a distributional critic to optimize actions, which are then distilled into the policy via supervised learning.
\item \textbf{D4PG~\cite{maron2018distributional}:} A distributional actor–critic baseline that combines multi-step returns with a deterministic policy gradient. We adapt it to the offline setting when necessary.

\end{itemize}
\subsection{Sources of Reported Results}\label{app:f2}
In the D4RL benchmark, we use a combination of previously reported results and results obtained from our own experiments.
Specifically, for PA-RL, FQL, and LEQ, we report the performance values as presented in the original papers that proposed each algorithm.
For IQL and ReBRAC, we use the results reported in the benchmark paper CORL~\cite{tarasov2023corl}.
In contrast, the performance of QC and D4PG is obtained from experiments conducted by us.
In particular, since D4PG is an online reinforcement learning algorithm, we implement it with an additional behavior cloning (BC) term for our experiments.

For the OGBench benchmark, we use the reported performance of HIQL from the original OGBench paper as the benchmark baseline~\cite{park2025ogbench}.
For FQL, we use the results reported in the paper that proposed FQL as well as those reported in Reinforcement Learning with Action Chunking~\cite{li2025qchunking}, which also introduces QC.
Additionally, from the same paper, we use the reported performance of other baselines on the scene-play and puzzle-3x3 tasks.
For all remaining tasks in OGBench, we conduct and report our own experimental results.

\section{Extensive Results of Performance Comparisons}\label{app:G}

\subsection{Learning Curve of VAN-Flow}\label{app:g1}
In this subsection, we present the learning curves of \textbf{\color{magenta}VAN-Flow} across representative benchmarks. Each curve shows the evaluation success rate as a function of training steps, averaged over five random seeds, and the shaded band around each curve indicates one standard deviation. Each run consists of an offline training phase followed by online fine-tuning. On OGBench, the agent is trained offline for the first $1$M steps and then fine-tuned online for the remaining $1$M steps. On D4RL, the switch occurs at $0.5$M steps. The vertical dashed line marks this switch, and the darker and lighter backgrounds indicate the offline and online phases, respectively. The corresponding offline-to-online results are summarized in Appendix~\ref{app:offon}. The learning curves for \texttt{antmaze} and \texttt{humanoidmaze} from OGBench are shown in Figure~\ref{fig:learning_curves_AO} and Figure~\ref{fig:learning_curves_HM}, respectively. Additionally, the learning curves for \texttt{puzzle} and \texttt{scene} are presented in Figure~\ref{fig:learning_curves_p} and Figure~\ref{fig:learning_curves_s}. Figure~\ref{fig:learning_curves_D} shows the learning curves for the \texttt{antmaze} tasks in D4RL.
\begin{figure*}[h]
    \centering
    \includegraphics[width=1.0\textwidth]{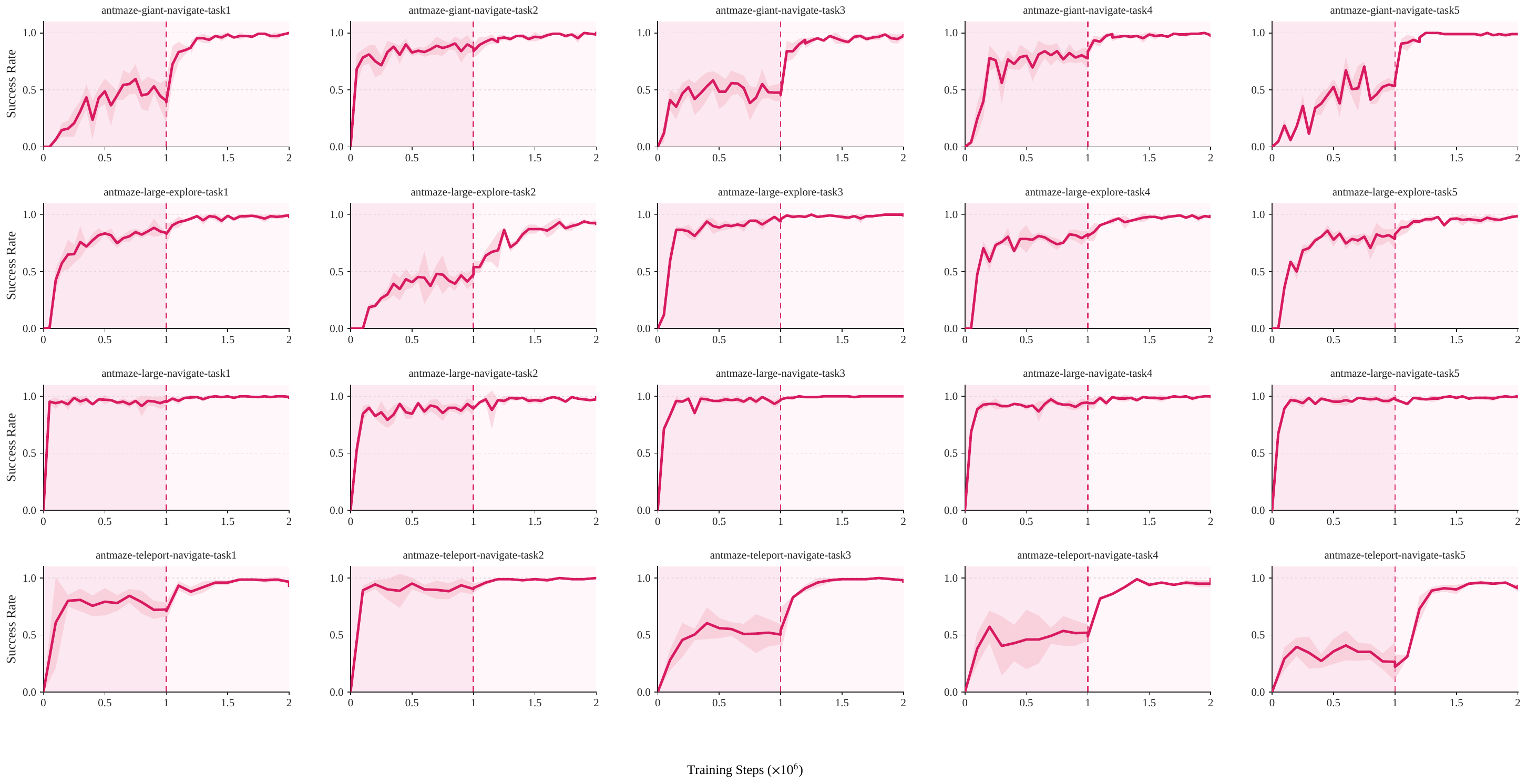}
    \caption{Learning curve on the \texttt{antmaze} task in \texttt{OGBench}}
    \label{fig:learning_curves_AO}
\end{figure*}
\begin{figure*}[h]
    \centering
    \includegraphics[width=1.0\textwidth]{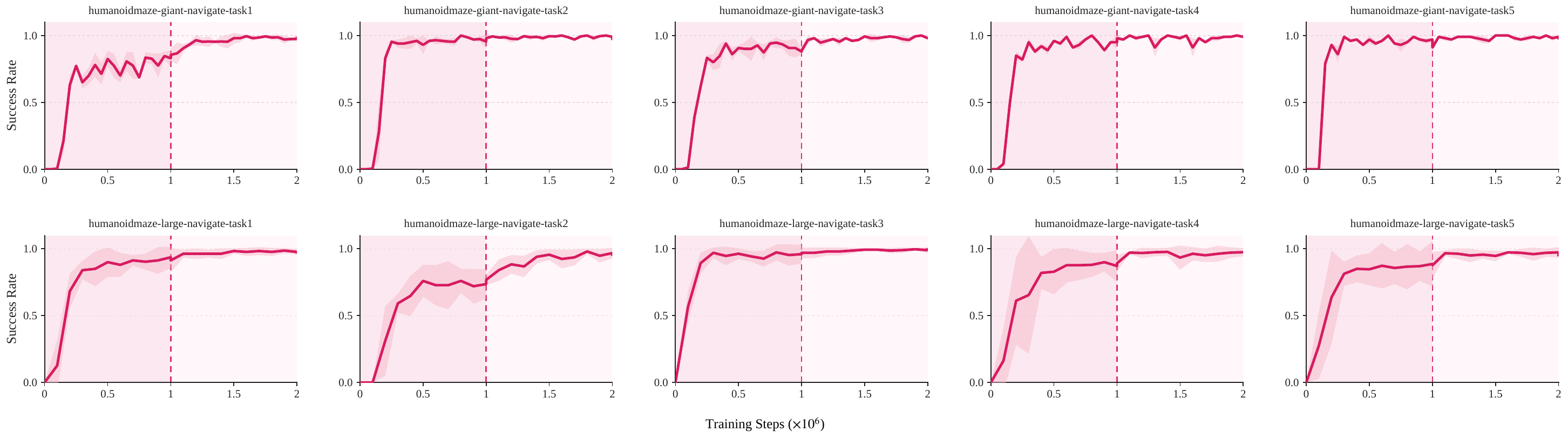}
    \caption{Learning curve on the \texttt{humanoidmaze} task in \texttt{OGBench}}
    \label{fig:learning_curves_HM}
\end{figure*}
\begin{figure*}[h]
    \centering
    \includegraphics[width=1.0\textwidth]{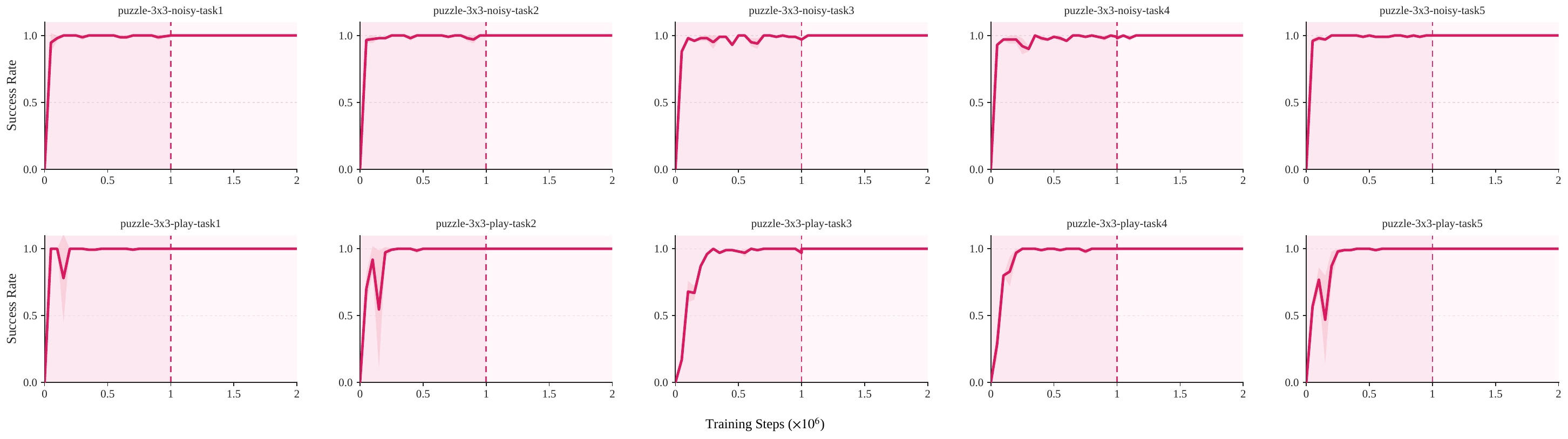}
    \caption{Learning curve on the \texttt{puzzle} task in \texttt{OGBench}}
    \label{fig:learning_curves_p}
\end{figure*}
\begin{figure*}[h]
    \centering
    \includegraphics[width=1.0\textwidth]{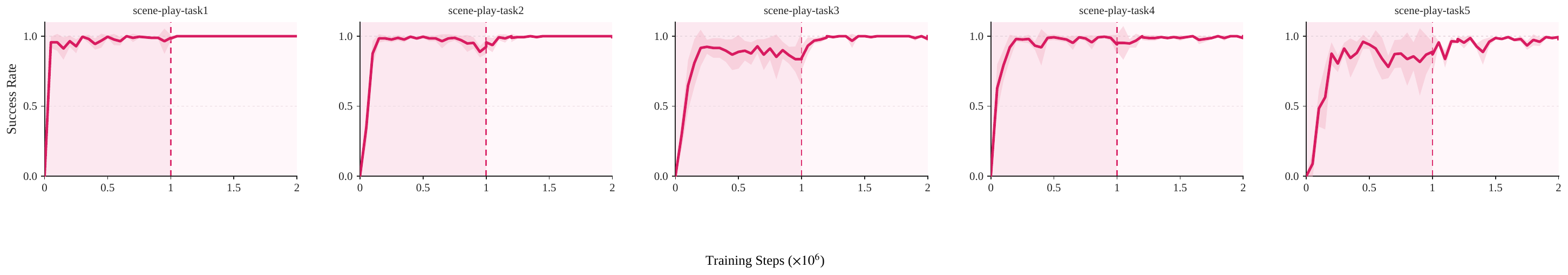}
    \caption{Learning curve on the \texttt{scene} task in \texttt{OGBench}}
    \label{fig:learning_curves_s}
\end{figure*}
\begin{figure*}[h]
    \centering
    \includegraphics[width=0.8\textwidth]{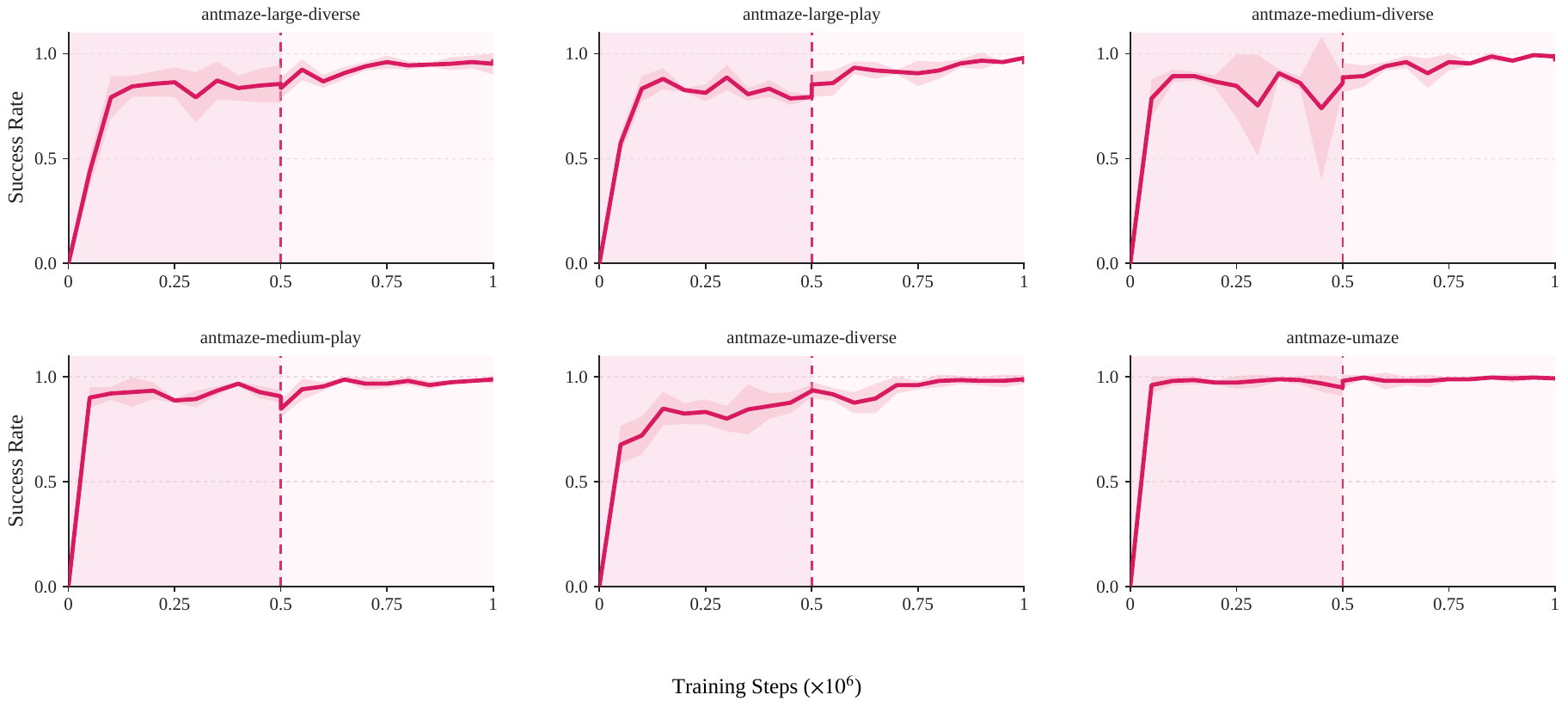}
    \caption{Learning curve on the \texttt{antmaze} task in \texttt{D4RL}}
    \label{fig:learning_curves_D}
\end{figure*}
\clearpage
\subsection{Full Performance Table}\label{app:g2}
Table~\ref{tab:task_wise_results_singletask} provides the full performance table for all evaluated tasks
and all compared baselines. We report normalized scores under the corresponding benchmark protocol.
\begin{table*}[h]
\centering
\footnotesize
\renewcommand{\arraystretch}{1.05}
\caption{Task-wise performance comparison in OGBench (singletask setting).}
\label{tab:task_wise_results_singletask}
\resizebox{\textwidth}{!}{%
\begin{tabular}{
l l|
cccc|
ccccc|
!{\color{magenta!50}\vrule width 0.9pt}
>{\columncolor{MagentaShade}}c
}
\hline

\rowcolor{white}
\multirow{2}{*}{\textbf{Env}}
& \multirow{2}{*}{\textbf{Task}}
& \multicolumn{4}{c|}{\textbf{\textit{Gaussian-based}}}
& \multicolumn{5}{c|}{\textbf{\textit{Flow-based}}}
& \textbf{\color{magenta}Proposed} \\

& 
& \textbf{IQL}
& \textbf{ReBRAC}
& \textbf{ReBRAC-$n$}
& \textbf{HIQL}
& \textbf{FQL}
& \textbf{BFN}
& \textbf{FQL-$n$}
& \textbf{BFN-$n$}
& \textbf{QC}
& \textbf{\color{magenta}VAN-Flow} \\
\hline

\multirow{5}{*}{\texttt{antmaze-large-navigate}}
& task1 & 48\pmv{9} & 91\pmv{10} & 64.0\pmv{25.8} & 93.0\pmv{3.0} & 80\pmv{8} & 93.0\pmv{1.4} & \textbf{94.2}\pmv{2.8} & 83.1\pmv{1.4} & 49.2\pmv{9.9} & \textbf{94.6}\pmv{4.2} \\
& task2 & 42\pmv{6} & \textbf{88}\pmv{4} & 78.3\pmv{5.9} & 78.0\pmv{9.0} & 57\pmv{10} & \textbf{90.2}\pmv{2.8} & 80.0\pmv{3.4} & 88.0\pmv{11.3} & 43.4\pmv{2.4} & \textbf{89.4}\pmv{7.0} \\
& task3 & 72\pmv{7} & 51\pmv{18} & 92.8\pmv{2.3} & 96.0\pmv{2.0} & 93\pmv{3} & 89.1\pmv{7.0} & \textbf{97.5}\pmv{1.4} & 74.5\pmv{5.6} & 71.8\pmv{2.8} & \textbf{98.0}\pmv{2.0} \\
& task4 & 51\pmv{9} & 84\pmv{7} & 90.0\pmv{2.8} & \textbf{94.0}\pmv{2.0} & 80\pmv{4} & 9.0\pmv{1.4} & 89.1\pmv{7.0} & 73.1\pmv{1.4} & 22.0\pmv{2.8} & \textbf{95.4}\pmv{3.1} \\
& task5 & 54\pmv{22} & 90\pmv{2} & 90.8\pmv{2.7} & 94.0\pmv{3.0} & 83\pmv{4} & 72.0\pmv{5.6} & 94.7\pmv{5.6} & 85.1\pmv{1.4} & 61.6\pmv{1.4} & \textbf{98.0}\pmv{2.0} \\
\hline

\multirow{5}{*}{\texttt{antmaze-teleport-navigate}}
& task1 & 17.3\pmv{1.5} & 0.0\pmv{0.0} & 14.0\pmv{12.6} & 37.0\pmv{5.0} & \textbf{39.0}\pmv{9.8} & \textbf{48.0}\pmv{2.8} & 5.0\pmv{1.4} & 23.2\pmv{1.4} & 37.0\pmv{1.2} & \textbf{70.8}\pmv{7.5} \\
& task2 & 27.0\pmv{4.2} & 0.0\pmv{0.0} & 38.4\pmv{21.2} & \textbf{66.0}\pmv{8.0} & 0.0\pmv{0.0} & 40.2\pmv{5.6} & 0.0\pmv{0.0} & 74.4\pmv{14.1} & 56.2\pmv{8.4} & \textbf{91.2}\pmv{3.0} \\
& task3 & 32.0\pmv{5.6} & 27.0\pmv{1.4} & 52.4\pmv{8.6} & 37.0\pmv{5.0} & 47.2\pmv{7.0} & 29.6\pmv{7.0} & \textbf{77.1}\pmv{4.2} & 41.0\pmv{9.9} & 37.0\pmv{3.5} & 54.0\pmv{14.1} \\
& task4 & 26.0\pmv{2.8} & 10.5\pmv{5.6} & 44.0\pmv{4.0} & 30.0\pmv{2.0} & 45.2\pmv{9.8} & 37.5\pmv{3.5} & 0.0\pmv{0.0} & 36.2\pmv{2.8} & 36.6\pmv{7.0} & \textbf{48.0}\pmv{5.5} \\
& task5 & 12.0\pmv{5.6} & 0\pmv{0} & 24.8\pmv{14.6} & 41.0\pmv{8.0} & 0\pmv{0} & \textbf{44.0}\pmv{2.8} & 27.5\pmv{1.4} & 29.2\pmv{1.4} & 31.0\pmv{4.2} & 22.0\pmv{8.0} \\
\hline

\multirow{5}{*}{\texttt{scene-play-v0}}
& task1 & 1.0 & 4.0 & 75.3\pmv{32.2} & 40.0\pmv{4.0} & 69.0 & 99.0 & 22.0 & 64.0 & \textbf{100} & 98.8\pmv{1.8} \\
& task2 & 0.0 & 51.0 & 76.7\pmv{23.0} & 40.4\pmv{9.0} & 51.0 & \textbf{99.0} & 4.0 & 15.0 & 98 & 92.4\pmv{4.5} \\
& task3 & 0.0 & 0.0 & 45.3\pmv{22.3} & 36.0\pmv{5.0} & 68.0 & \textbf{93.0} & 1.0 & 34.0 & 90 & 87.6\pmv{7.9} \\
& task4 & 0.0 & 0.0 & 4.0\pmv{2.5} & 55.0\pmv{5.0} & 75.0 & 86.0 & 50.0 & 91.0 & 93.0 & \textbf{95.2}\pmv{4.6} \\
& task5 & 0.0 & 11.0 & 0.0\pmv{0.0} & 20.0\pmv{5.0} & 25.0 & 38.0 & 14.0 & 71.0 & 43.0 & \textbf{88.8}\pmv{4.6} \\
\hline

\multirow{5}{*}{\texttt{puzzle-3x3-play}}
& task1 & 0.0 & 96.0 & 48.8\pmv{5.4} & 29.0\pmv{4.0} & 99 & 98 & \textbf{100.0} & 88 & \textbf{100.0} & \textbf{100.0}\pmv{0.0} \\
& task2 & 0.0 & 70.0 & 2.8\pmv{1.8} & 11.0\pmv{3.0} & 99 & 98 & \textbf{100.0} & 72 & \textbf{100.0} & \textbf{100.0}\pmv{0.0} \\
& task3 & 0.0 & 27.0 & 0.0\pmv{0.0} & 7.0\pmv{3.0} & \textbf{100.0} & 95 & \textbf{100.0} & 56 & \textbf{100.0} & \textbf{100.0}\pmv{0.0} \\
& task4 & 0.0 & 29.0 & 0.4\pmv{0.9} & 5.0\pmv{1.0} & \textbf{100.0} & 98 & \textbf{100.0} & 37 & \textbf{100.0} & \textbf{100.0}\pmv{0.0} \\
& task5 & 0.0 & 54.0 & 0.0\pmv{0.0} & 8.0\pmv{3.0} & \textbf{100.0} & 97 & 98 & 47 & \textbf{100.0} & \textbf{100.0}\pmv{0.0} \\
\hline

\multicolumn{12}{l}{\textbf{Long-horizon tasks}} \\
\hline

\multirow{5}{*}{\texttt{antmaze-giant-navigate}}
& task1 & 0\pmv{0} & 27\pmv{22} & 0.0\pmv{0.0} & 47.0\pmv{10.0} & 4\pmv{5} & 0.0\pmv{0.0} & 11.2\pmv{1.6} & 0.0\pmv{0.0} & 0.0\pmv{0.0} & \textbf{58.0}\pmv{4.0} \\
& task2 & 1\pmv{1} & 16\pmv{17} & 81.6\pmv{3.0} & 74.0\pmv{5.0} & 9\pmv{7} & 0.0\pmv{0.0} & \textbf{93.6}\pmv{7.0} & 16.1\pmv{8.4} & 4.0\pmv{2.8} & 86.6\pmv{7.0} \\
& task3 & 0\pmv{0} & 34\pmv{22} & 0.0\pmv{0.0} & \textbf{55.0}\pmv{7.0} & 0\pmv{1} & 0.0\pmv{0.0} & 44.0\pmv{8.4} & 0.0\pmv{0.0} & 0.0\pmv{0.0} & \textbf{56.4}\pmv{14.2} \\
& task4 & 0\pmv{0} & 5\pmv{12} & 0.0\pmv{0.0} & 69.0\pmv{5.0} & 14\pmv{23} & 0.0\pmv{0.0} & 57.5\pmv{12.7} & 0.0\pmv{0.0} & 0.0\pmv{0.0} & \textbf{87.2}\pmv{7.5} \\
& task5 & 19\pmv{7} & 49\pmv{22} & 0.0\pmv{0.0} & \textbf{82.0}\pmv{4.0} & 16\pmv{28} & 55.2\pmv{4.7} & 63.6\pmv{1.4} & 7.2\pmv{4.2} & 6.2\pmv{2.8} & 56.8\pmv{4.1} \\
\hline

\multirow{5}{*}{\texttt{humanoidmaze-large-navigate}}
& task1 & 3\pmv{1} & 2\pmv{1} & 0.8\pmv{1.1} & 67.0\pmv{4.0} & 7\pmv{6} & 44.2\pmv{5.6} & 25.3\pmv{1.4} & 45.2\pmv{2.4} & 13.0\pmv{1.4} & \textbf{87.4}\pmv{5.8} \\
& task2 & 0\pmv{0} & 0\pmv{0} & 0.0\pmv{0.0} & 2.0\pmv{3.0} & 0\pmv{0} & 7.2\pmv{1.4} & 0.0\pmv{0.0} & 11.5\pmv{1.4} & 0.0\pmv{0.0} & \textbf{55.4}\pmv{8.0} \\
& task3 & 7\pmv{3} & 2\pmv{1} & 18.5\pmv{8.5} & 88.0\pmv{3.0} & 7\pmv{6} & 26.3\pmv{2.8} & 41.0\pmv{9.8} & 65.2\pmv{9.8} & 30.0\pmv{2.8} & \textbf{96.0}\pmv{2.0} \\
& task4 & 1\pmv{0} & 1\pmv{1} & 0.0\pmv{0.0} & 42.0\pmv{11.0} & 2\pmv{3} & 3.1\pmv{1.4} & 24.5\pmv{2.8} & 45.1\pmv{2.8} & 15.0\pmv{1.4} & \textbf{84.6}\pmv{1.1} \\
& task5 & 1\pmv{1} & 2\pmv{2} & 0.0\pmv{0.0} & 47.0\pmv{10.0} & 1\pmv{3} & 52.3\pmv{11.1} & 10.0\pmv{2.8} & 53.5\pmv{9.9} & 23.2\pmv{2.8} & \textbf{82.6}\pmv{3.0} \\
\hline

\multirow{5}{*}{\texttt{humanoidmaze-giant-navigate}}
& task1 & 0.0\pmv{0.0} & 0.0\pmv{0.0} & 0.0\pmv{0.0} & 13.0\pmv{7.0} & 0.0\pmv{0.0} & 0.0\pmv{0.0} & 0.0\pmv{0.0} & 54.6\pmv{2.8} & 10.2\pmv{2.7} & \textbf{85.4}\pmv{5.0} \\
& task2 & 0.0\pmv{0.0} & 0.0\pmv{0.0} & 0.0\pmv{0.0} & 35.0\pmv{11.0} & 1.2\pmv{1.4} & 0.0\pmv{0.0} & 7.6\pmv{1.4} & 87.1\pmv{5.6} & 3.7\pmv{1.5} & \textbf{97.4}\pmv{1.1} \\
& task3 & 0.0\pmv{0.0} & 0.0\pmv{0.0} & 0.0\pmv{0.0} & 11.0\pmv{4.0} & 0.0\pmv{0.0} & 0.0\pmv{0.0} & 0.0\pmv{0.0} & 70.1\pmv{2.8} & 3.0\pmv{1.2} & \textbf{88.0}\pmv{2.0} \\
& task4 & 0.0\pmv{0.0} & 0.0\pmv{0.0} & 0.0\pmv{0.0} & 2.0\pmv{2.0} & 0.0\pmv{0.0} & 0.0\pmv{0.0} & 0.0\pmv{0.0} & 73.5\pmv{7.0} & 0.0\pmv{0.0} & \textbf{96.6}\pmv{3.0} \\
& task5 & 3.6\pmv{3.5} & 0.0\pmv{0.0} & 0.0\pmv{0.0} & 12.0\pmv{4.0} & 0.0\pmv{0.0} & 0.0\pmv{0.0} & 9.0\pmv{7.0} & 86.1\pmv{2.8} & 44.2\pmv{8.4} & \textbf{97.2}\pmv{1.1} \\
\hline

\multicolumn{12}{l}{\textbf{Noisy tasks}} \\
\hline

\multirow{5}{*}{\texttt{antmaze-large-explore}}
& task1 & 0.0\pmv{0.0} & 90.0\pmv{5.6} & 24.8\pmv{10.4} & 1.0\pmv{3.0} & 2.0\pmv{1.4} & 17.0\pmv{24.0} & 0.0\pmv{0.0} & 5.2\pmv{1.4} & 79.2\pmv{4.2} & \textbf{93.6}\pmv{3.0} \\
& task2 & 0.0\pmv{0.0} & 0.0\pmv{0.0} & 0.0\pmv{0.0} & 0.0\pmv{0.0} & 0.0\pmv{0.0} & 2.0\pmv{2.8} & 0.0\pmv{0.0} & 0.0\pmv{0.0} & 7.0\pmv{4.2} & \textbf{52.0}\pmv{3.4} \\
& task3 & 67.5\pmv{1.4} & 93.4\pmv{1.4} & 95.6\pmv{2.2} & 18.0\pmv{24.0} & 87.6\pmv{4.2} & 89.3\pmv{1.4} & 4.5\pmv{3.5} & 77.2\pmv{4.2} & 90.2\pmv{1.4} & \textbf{96.0}\pmv{4.0} \\
& task4 & 2.0\pmv{1.0} & 0.5\pmv{0.7} & 46.8\pmv{13.2} & 0.0\pmv{0.0} & 10.2\pmv{2.8} & 0.0\pmv{0.0} & 0.0\pmv{0.0} & 0.0\pmv{0.0} & 64.2\pmv{2.8} & \textbf{86.6}\pmv{3.7} \\
& task5 & 6.0\pmv{2.8} & 58.5\pmv{14.5} & 1.6\pmv{0.9} & 0.0\pmv{0.0} & 0.0\pmv{0.0} & 1.0\pmv{1.4} & 1.0\pmv{1.4} & 1.0\pmv{1.4} & 48.6\pmv{7.0} & \textbf{92.2}\pmv{3.0} \\
\hline

\multirow{5}{*}{\texttt{puzzle-3x3-noisy}}
& task1 & 34.5\pmv{1.5} & 9.0\pmv{1.4} & 18.8\pmv{39.8} & 67.0\pmv{10.0} & \textbf{100.0}\pmv{0.0} & \textbf{100.0}\pmv{0.0} & 44.1\pmv{2.2} & \textbf{100.0}\pmv{0.0} & \textbf{100.0}\pmv{0.0} & 96.6\pmv{3.0} \\
& task2 & 53.4\pmv{3.5} & 2.3\pmv{2.8} & 19.6\pmv{41.6} & 54.0\pmv{11.0} & \textbf{100.0}\pmv{0.0} & \textbf{100.0}\pmv{0.0} & 6.0\pmv{8.4} & \textbf{100.0}\pmv{0.0} & \textbf{100.0}\pmv{0.0} & 98.6\pmv{4.1} \\
& task3 & 19.0\pmv{1.4} & 0.0\pmv{0.0} & 0.0\pmv{0.0} & 43.0\pmv{12.0} & \textbf{100.0}\pmv{0.0} & \textbf{100.0}\pmv{0.0} & 84.0\pmv{2.8} & 99.0\pmv{1.4} & \textbf{100.0}\pmv{0.0} & 95.3\pmv{3.0} \\
& task4 & 20.3\pmv{0.7} & 1.5\pmv{1.4} & 0.0\pmv{0.0} & 41.0\pmv{13.0} & \textbf{100.0}\pmv{0.0} & \textbf{100.0}\pmv{0.0} & 18.0\pmv{8.4} & 93.0\pmv{4.2} & \textbf{100.0}\pmv{0.0} & 96.6\pmv{3.0} \\
& task5 & 32.1\pmv{13.8} & 0.0\pmv{0.0} & 0.0\pmv{0.0} & 47.0\pmv{15.0} & \textbf{100.0}\pmv{0.0} & \textbf{100.0}\pmv{0.0} & 13.0\pmv{4.2} & \textbf{100.0}\pmv{0.0} & \textbf{100.0}\pmv{0.0} & 97.4\pmv{1.8} \\
\hline

\end{tabular}
}
\end{table*}

\clearpage
\subsection{Performance Comparison in Denser Reward Environments}\label{app:g3}
We additionally evaluate \textbf{\color{magenta}VAN-Flow} on the D4RL Adroit manipulation suite to
compare against a wide range of generative policy-based offline RL methods,
including diffusion-based approaches (IDQL~\cite{hansen2023dql}, SRPO~\cite{chen2024score}, CAC~\cite{ding2024consistency}) and flow-based
policies (FAWAC, FBRAC, FQL, IFQL).

For flow-based baselines, we consider flow-policy instantiations of existing
offline RL algorithms (AWAC~\cite{nair2020awac}, ReBRAC~\cite{tarasov2023revisiting}, IDQL)~\cite{hansen2023dql} as implemented in
Flow Q-Learning~\cite{park2025fql}.

As shown in Table~\ref{tab:adroit_results}, \textbf{\color{magenta}VAN-Flow} achieves competitive or superior performance
across multiple Adroit tasks, particularly on cloned and expert datasets.
Notably, VAN-Flow outperforms flow-based baselines such as FQL
and IFQL on most tasks, indicating that the observed gains are not merely due to the use of
a generative policy, but arise from variance-aware $n$-step value learning.

Furthermore, \textbf{\color{magenta}VAN-Flow} remains competitive with strong diffusion-based methods,
demonstrating that variance-averse n-step learning integrates naturally with
flow-based policies without sacrificing the advantages typically associated
with diffusion models.
\begin{table*}[h]
\centering
\footnotesize
\renewcommand{\arraystretch}{1.1}
\caption{Performance comparison in D4RL Adroit}
\label{tab:adroit_results}
\resizebox{0.97\textwidth}{!}{%
\begin{tabular}{
l|
>{\columncolor{GrayShade}}c
>{\columncolor{GrayShade}}c
>{\columncolor{GrayShade}}c|
>{\columncolor{GrayShade}}c
>{\columncolor{GrayShade}}c
>{\columncolor{GrayShade}}c|
>{\columncolor{GrayShade}}c
>{\columncolor{GrayShade}}c
>{\columncolor{GrayShade}}c
>{\columncolor{GrayShade}}c
>{\columncolor{MagentaShade}}c
}
\hline

\rowcolor{white}
\multirow{2}{*}{\textbf{Task}}
& \multicolumn{3}{c|}{\textbf{\textit{Gaussian-based}}}
& \multicolumn{3}{c|}{\textbf{\textit{Diffusion Policies}}}
& \multicolumn{4}{c|}{\textbf{\texttt{Flow-based}}}
& \textbf{Proposed} \\

& \textbf{BC}
& \textbf{IQL}
& \textbf{ReBRAC}
& \textbf{IDQL}
& \textbf{SRPO}
& \textbf{CAC}
& \textbf{FAWAC}
& \textbf{FBRAC}
& \textbf{IFQL}
& \textbf{FQL}
& \textbf{Ours} \\
\hline

\texttt{pen-human-v1}
& 71 & 78 & \textbf{103} & 76 & 69 & 64 & 67 & 77 & 71 & 53 & 82.3\pmv{3.5} \\

\texttt{pen-cloned-v1}
& 52 & 83 & 103 & 64 & 61 & 56 & 62 & 67 & 80 & 74 & \textbf{107.5}\pmv{14.9} \\

\texttt{pen-expert-v1}
& 110 & 128 & \textbf{152} & 140 & 134 & 103 & 118 & 119 & 139 & 142 & 146.6\pmv{5.5} \\
\hline

\texttt{door-human-v1}
& 2 & 3 & 0 & \textbf{6} & 3 & 5 & 2 & 4 & 7 & 0 & 2.7\pmv{1.3} \\

\texttt{door-cloned-v1}
& 0 & \textbf{3} & 0 & 0 & 0 & 1 & 0 & 0 & 2 & 2 & \textbf{3.0}\pmv{2.8} \\

\texttt{door-expert-v1}
& 105 & \textbf{107} & 106 & 105 & 105 & 98 & 103 & 104 & 104 & 104 & 104.0\pmv{0.6} \\
\hline

\texttt{hammer-human-v1}
& 3 & 2 & 0 & 2 & 1 & 2 & 2 & 2 & 3 & 1 & \textbf{3.5}\pmv{0.7} \\

\texttt{hammer-cloned-v1}
& 1 & 2 & 5 & 2 & 2 & 1 & 1 & 2 & 2 & \textbf{11} & 3.0\pmv{2.1} \\

\texttt{hammer-expert-v1}
& 127 & 129 & \textbf{134} & 125 & 127 & 92 & 118 & 119 & 117 & 125 & 131\pmv{1.4} \\
\hline

\texttt{relocate-human-v1}
& 0 & 0 & 0 & 0 & 0 & 0 & 0 & 0 & 0 & 0 & 0.2\pmv{0.0} \\

\texttt{relocate-cloned-v1}
& 0 & 0 & 2 & 0 & 0 & 0 & 0 & 1 & 0 & 0 & 0.3\pmv{0.0} \\

\texttt{relocate-expert-v1}
& \textbf{108} & 106 & \textbf{108} & 107 & 106 & 93 & 105 & 105 & 104 & 107 & \textbf{108}\pmv{0.8} \\
\hline

\end{tabular}
}
\end{table*}

\subsection{D4RL Locomotion Tasks}\label{app:g4}
The main experiments focus on sparse-reward, long-horizon tasks because that is where variance amplification under $n$-step returns is most severe. To check that the method is not tied to goal-conditioned tasks, we also ran \textbf{\color{magenta}VAN-Flow} on the dense-reward D4RL locomotion suite with the \texttt{medium} and \texttt{medium-replay} datasets. Neither dataset contains expert demonstrations. We keep the default $\delta=2$ and use $\gamma=0.99$, $I=101$, and the minimum over the two critics, with $n=1$ for HalfCheetah and $n=3$ for Hopper and Walker2d, and $\lambda=0.3$, $3$, and $10$ for the three environments respectively. Results are averaged over five seeds, and baseline numbers are taken from the corresponding papers.

\begin{table}[h]
\centering
\footnotesize
\caption{Normalized returns on D4RL locomotion tasks. Best results are in bold, and results within one standard deviation of the best are also bolded.}
\label{tab:locomotion}
\renewcommand{\arraystretch}{1.1}
\setlength{\tabcolsep}{6pt}
\resizebox{\linewidth}{!}{%
\begin{tabular}{lccccc}
\toprule
\textbf{Task} & \textbf{IQL} & \textbf{ReBRAC} & \textbf{Diffusion-QL} & \textbf{Diffuser} & \textbf{\color{magenta}VAN-Flow} \\
\midrule
\texttt{halfcheetah-medium}        & 47.4 & $\mathbf{65.6\pmv{1.0}}$ & 51.1\pmv{0.5}  & 42.8\pmv{0.3} & 63.9\pmv{2.1} \\
\texttt{halfcheetah-medium-replay} & 44.2 & 51.0\pmv{0.8}            & 47.8\pmv{0.3}  & 37.7\pmv{0.5} & $\mathbf{51.4\pmv{0.3}}$ \\
\texttt{hopper-medium}             & 66.3 & $\mathbf{102.0\pmv{1.0}}$ & 90.5\pmv{4.6} & 74.3\pmv{1.4} & 95.3\pmv{2.5} \\
\texttt{hopper-medium-replay}      & 94.7 & 98.1\pmv{5.3}            & $\mathbf{101.3\pmv{0.6}}$ & 93.6\pmv{0.4} & 100.3\pmv{1.2} \\
\texttt{walker2d-medium}           & 78.3 & 82.5\pmv{3.6}            & 87.0\pmv{0.9}  & 79.6\pmv{0.6} & $\mathbf{88.2\pmv{1.8}}$ \\
\texttt{walker2d-medium-replay}    & 73.9 & 77.3\pmv{7.9}            & $\mathbf{95.5\pmv{1.5}}$ & 70.6\pmv{1.6} & 81.0\pmv{4.1} \\
\midrule
\textbf{Mean}                      & 67.5 & 79.4 & 78.9 & 66.4 & $\mathbf{80.0}$ \\
\bottomrule
\end{tabular}}
\end{table}

VAN-Flow is competitive with the strongest baselines here (mean 80.0 against 79.4 for ReBRAC and 78.9 for Diffusion-QL). It matches or outperforms ReBRAC on all three \texttt{medium-replay} datasets, which mix behavior of very different quality, while the gap is small on the more homogeneous \texttt{medium} datasets. This is consistent with what the method is designed for. When several behavior modes with different returns coexist, return variance separates reliable from unreliable actions, and when the dataset comes from a single behavior policy, variance adds little beyond the expected return and the operator behaves close to neutral.

\section{Off-to-Online Fine-Tuning Results}\label{app:offon}

This appendix presents off-to-online fine-tuning results for the off-to-online evaluation. We compare \textbf{\color{magenta}VAN-Flow} against QC~\cite{li2025qchunking} across long-horizon, noisy, and D4RL AntMaze tasks. Scores in gray indicate performance before online fine-tuning (the offline scores of Tables~\ref{tab:task_wise_results} and~\ref{tab:antmaze_results}), and bold scores denote performance after fine-tuning.
\begin{table}[H]
\centering
\footnotesize
\caption{Offline-to-online fine-tuning performance across diverse long-horizon and noisy control tasks.
Scores are reported before online fine-tuning in gray and after fine-tuning in bold. The offline scores are those reported in Tables~\ref{tab:task_wise_results} and~\ref{tab:antmaze_results}.}

\resizebox{0.9\textwidth}{!}{%
\begin{tabular}{lcc|lcc}
\toprule
\textbf{Task} & \textbf{\color{magenta}VAN-Flow} & \textbf{QC}
& \textbf{Task} & \textbf{\color{magenta}VAN-Flow} & \textbf{QC} \\
\midrule

\texttt{antmaze-teleport-navigate} & \offon{57}{96} & \offon{40}{46} 
& \texttt{scene-play}              & \offon{93}{100} & \offon{84}{99}  \\

\texttt{antmaze-large-navigate}    & \offon{95}{99} & \offon{49}{98}
   
& \texttt{puzzle-3x3-play}         & \offon{100}{100} & \offon{100}{100} \\

\midrule
\multicolumn{3}{l}{\textbf{Long-horizon tasks}}
& \multicolumn{3}{l}{\textbf{Noisy tasks}} \\
\midrule

\texttt{humanoidmaze-large-navigate} & \offon{81}{98} & \offon{16}{16}
& \texttt{antmaze-large-explore}     & \offon{84}{98} & \offon{58}{96} \\

\texttt{humanoidmaze-giant-navigate} & \offon{92}{98} & \offon{12}{67}
& \texttt{puzzle-3x3-noisy}          & \offon{97}{100} & \offon{100}{100} \\

\texttt{antmaze-giant-navigate}      & \offon{68}{98} & \offon{2}{69}
& \texttt{ }     &   &   \\

\midrule
\multicolumn{3}{l}{\textbf{D4RL AntMaze tasks}}
& \multicolumn{3}{l}{} \\
\midrule

\texttt{antmaze-umaze}          & \offon{98.0}{96} & \offon{95.0}{98}
& \texttt{antmaze-medium-diverse} & \offon{89.3}{94} & \offon{64.0}{95} \\

\texttt{antmaze-umaze-diverse}  & \offon{95.6}{98} & \offon{78.0}{92}
& \texttt{antmaze-large-play}    & \offon{86.5}{95} & \offon{80.0}{94} \\

\texttt{antmaze-medium-play}    & \offon{90.7}{98} & \offon{72.0}{98}
& \texttt{antmaze-large-diverse} & \offon{88.6}{98} & \offon{82.0}{93} \\

\bottomrule
\end{tabular}
}
\label{tab:offline_to_online_twocol}
\vspace{-0.1cm}
\end{table}

\paragraph{Discussion.}
Off-to-online evaluation tests whether offline-trained policies serve as strong initializations for online fine-tuning~\cite{11080103}. We perform online post-training from offline agents and compare with QC~\cite{li2025qchunking}, a state-of-the-art horizon-reduction method, to evaluate practical adaptability. As shown in Table~\ref{tab:offline_to_online_twocol}, \textbf{\color{magenta}VAN-Flow} adapts stably online and improves with more interaction on every task that is not already saturated offline. Notably, it matches or slightly surpasses QC in several settings, indicating that variance-averse learning provides high-quality initializations that transfer smoothly from offline to online RL.

\section{Additional Comparisons with Risk-Sensitive Objectives}
\label{app:risk}

We additionally compare the proposed \textbf{\color{magenta}variance-averse expectation} with representative risk-sensitive objectives. To this end, we replace $\mathcal{E}(Z)$ in~\eqref{eq:rejection_sampling} and~\eqref{eq:pro_actor_loss} with alternative aggregation rules, while keeping all other components unchanged.

We consider $\mathbb{E}[Z]$, CVaR, CPW, Norm, Wang, entropic risk, Sharpe ratio, Sortino ratio, and mean--variance. The comparisons are conducted on \texttt{antmaze-teleport-navigate} and \texttt{antmaze-large-explore}, which represent two challenging regimes for variance-aware learning.

As shown in Table~\ref{tab:risk_compare}, \textbf{\color{magenta}VE} consistently outperforms these alternatives across both tasks. Tail-focused objectives such as CVaR can become overly conservative, while scalar penalty-based objectives such as mean--variance or entropic risk do not fully exploit the categorical return structure. By contrast, $\mathcal{E}(Z)$ applies a smooth CDF-based reweighting over the full return distribution, which better aligns with our goal of suppressing unreliable high-variance targets without discarding distributional information.

\begin{table}[H]
\centering
\footnotesize
\caption{
Comparison with representative risk-sensitive objectives.
Scores denote success rates averaged over five random seeds.
}
\label{tab:risk_compare}
\setlength{\tabcolsep}{3.5pt}
\renewcommand{\arraystretch}{1.05}
\resizebox{0.98\linewidth}{!}{
\begin{tabular}{lcccccccccc}
\toprule
\textbf{Task} & $\boldsymbol{\mathbb{E}[Z]}$ & \textbf{CVaR} & \textbf{CPW} & \textbf{Norm} & \textbf{Wang} & \textbf{Entropic} & \textbf{Sharpe} & \textbf{Sortino} & \textbf{Mean--Var.} & \textbf{\color{magenta}VE} \\
\midrule
\texttt{teleport} & $38\!\pm\!7$ & $54\!\pm\!16$ & $37\!\pm\!5$ & $36\!\pm\!9$ & $40\!\pm\!10$ & $44\!\pm\!16$ & $25\!\pm\!11$ & $30\!\pm\!8$ & $46\!\pm\!8$ & $\mathbf{71\!\pm\!7}$ \\
\texttt{explore}  & $84\!\pm\!2$ & $2\!\pm\!1$   & $87\!\pm\!4$ &  $84\!\pm\!8$ & $86\!\pm\!2$  & $86\!\pm\!4$ & $64\!\pm\!27$ & $70\!\pm\!8$ & $85\!\pm\!4$ & $\mathbf{93\!\pm\!3}$ \\
\bottomrule
\end{tabular}
}
\end{table}

\section{Full Results on Sensitivity Analyses}\label{app:H}

We provide extended sensitivity results for the key hyperparameters $n$, $M$, and $K$. All plots report mean $\pm$ standard deviation over random seeds.

\subsection{$n$-step Horizon}\label{app:h1}
\begin{figure}[H]
    \centering
    \includegraphics[width=0.7\linewidth]{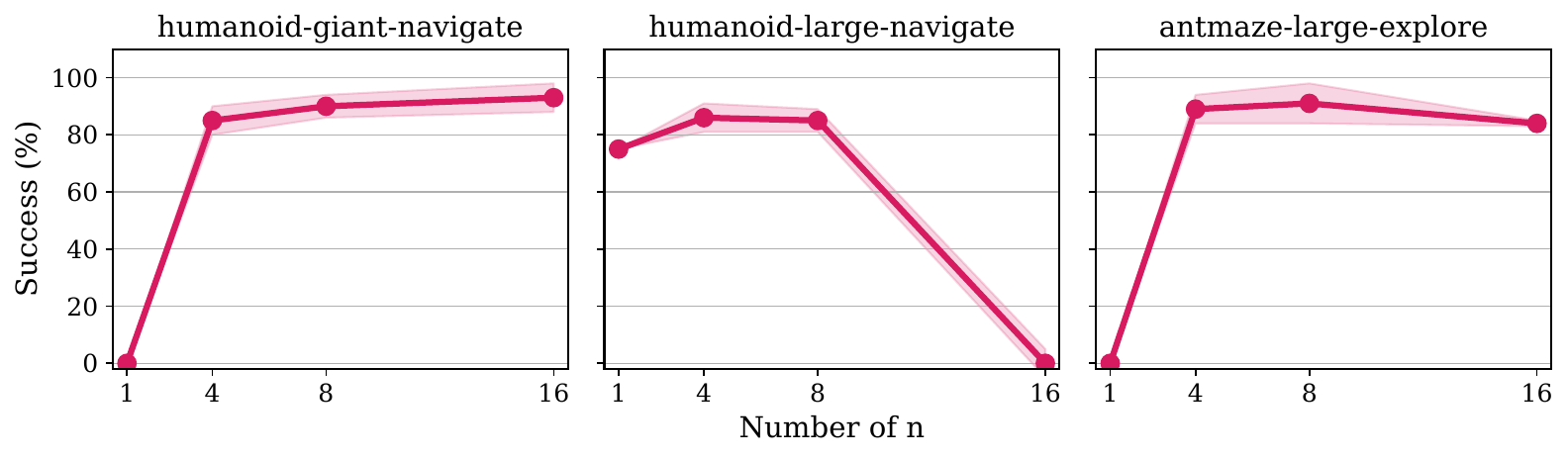}
    \caption{Additional sensitivity results on the number of $n$.}
    \label{fig:nstep}
\end{figure}
We sweep $n$ over a task-dependent grid (e.g., $n\in\{1,4,8,16\}$).
For long-horizon tasks, increasing $n$ initially improves performance by reducing bootstrapping bias,
but overly large $n$ can cause instability due to amplified return variance.
\textbf{\color{magenta}VAN-Flow} remains robust across a wide range of $n$, indicating that variance-aware
aggregation mitigates the typical sensitivity of long-horizon targets.

\subsection{Number of Rejection Samples}\label{app:h2}
\begin{figure}[H]
    \centering
    \includegraphics[width=0.7\linewidth]{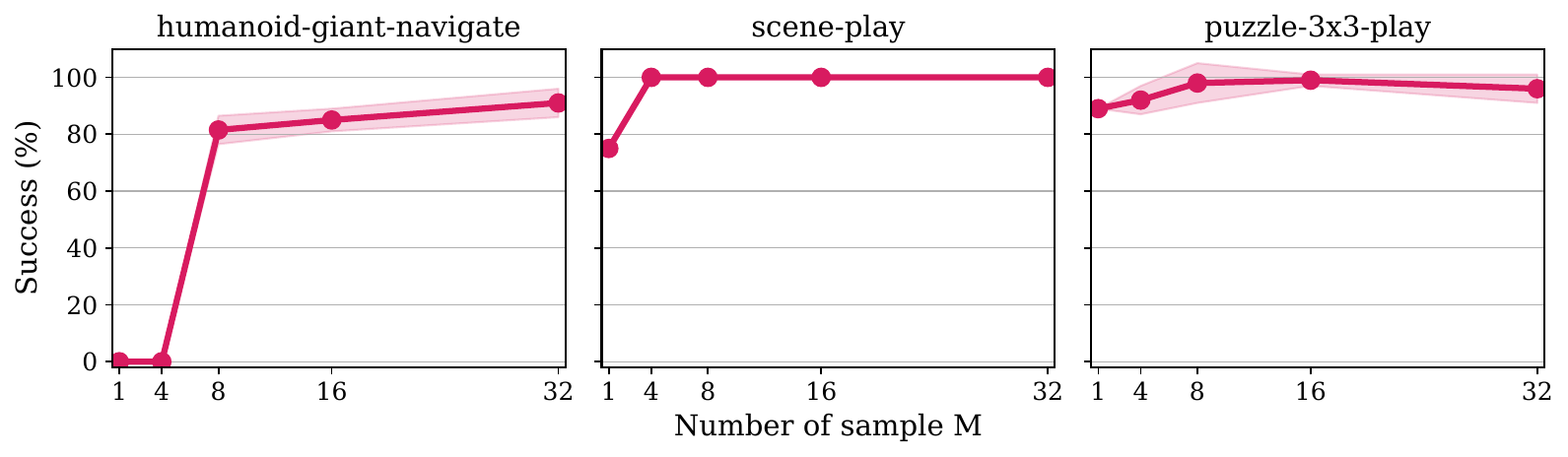}
    \caption{Additional sensitivity results on the number of rejection samples}
    \label{fig:samples}
\end{figure}

We vary the number of rejection candidates $M$. Performance generally improves with $M$ and quickly saturates, so strong performance is achieved without excessive sampling.

\subsection{Number of Flow Steps}\label{app:h3}
\begin{figure}[H]
    \centering
    \includegraphics[width=0.7\linewidth]{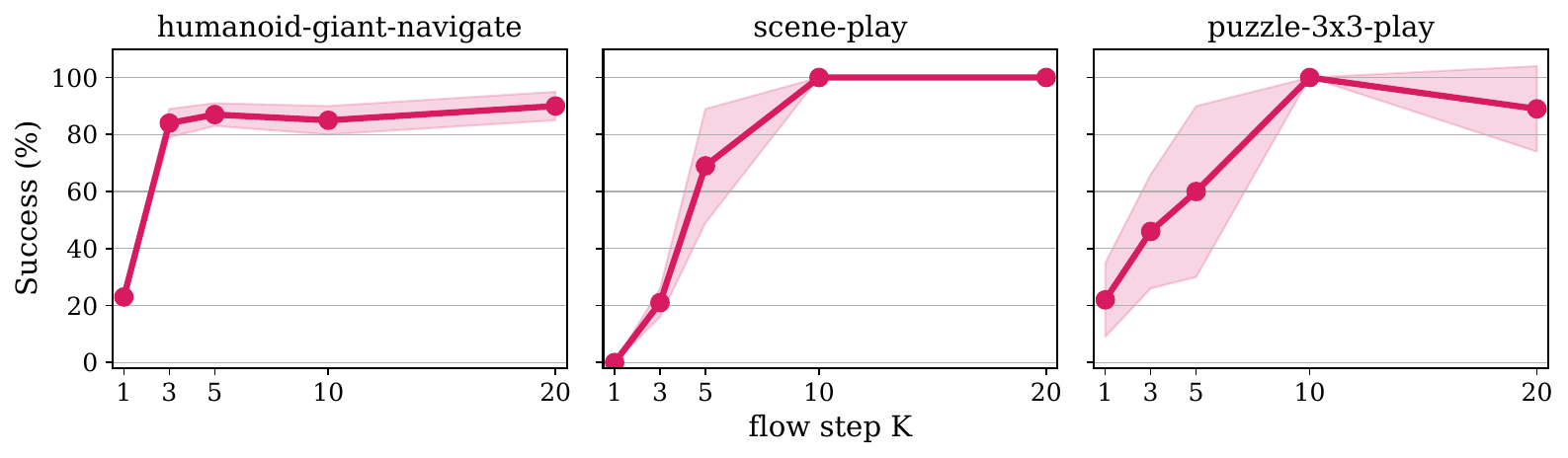}
    \caption{Additional sensitivity results on the number of flow steps.}
    \label{fig:flow_steps}
\end{figure}
We vary the number of flow discretization steps $K$ (e.g., $K\in\{1,3,5,10,20\}$).
A single step is not enough on any task. The maze tasks are already close to their best with three to five steps, whereas \texttt{scene-play} and \texttt{puzzle-3x3-play} need about ten steps, and going beyond that brings little. On \texttt{humanoidmaze-large}, Table~\ref{tab:flow_vs_diffusion} even shows a mild decline for large $K$. Precise ODE integration is therefore not required, and $K=10$ is a safe default across tasks.

\subsection{Aversion Coefficient $\delta$}\label{app:h4}
Table~\ref{tab:delta_sweep} sweeps $\delta$ on the first task of three environments. Too little aversion leaves the variance unsuppressed, and too much aversion makes the policy overly conservative on some tasks, most visibly on \texttt{antmaze-large-explore}, where the success rate collapses for $\delta\ge5$. The default $\delta=2$ is at or near the best value on all three tasks and avoids this collapse, which is why we use it as the default. The one exception is \texttt{antmaze-teleport-navigate}, where the stochastic teleporters put irreducible variance into multi-step targets and a stronger aversion ($\delta=7$, Table~\ref{tab:vanflow_hparams}) works better. The trend matches Theorem~\ref{thm:VaE_mps}: the spectral form interpolates between the plain expectation at $\delta=0$ and worst-case selection as $\delta\to\infty$.

\begin{table}[H]
\centering
\footnotesize
\caption{Success rate for different values of the aversion coefficient $\delta$ (task1 of each environment, five seeds).}
\label{tab:delta_sweep}
\renewcommand{\arraystretch}{1.1}
\setlength{\tabcolsep}{5pt}
\resizebox{\linewidth}{!}{%
\begin{tabular}{lccccccc}
\toprule
\textbf{Task} & $\delta=0$ & $1$ & $2$ (default) & $3$ & $4$ & $5$ & $10$ \\
\midrule
\texttt{antmaze-giant-navigate}  & 44.0\pmv{11.8} & 53.6\pmv{6.9} & $\mathbf{58.0\pmv{4.0}}$ & 30.0\pmv{3.3} & 29.5\pmv{6.7} & 31.6\pmv{4.1} & 30.4\pmv{5.9} \\
\texttt{antmaze-large-explore}   & 82.0\pmv{8.6}  & 89.5\pmv{3.4} & $\mathbf{93.6\pmv{3.0}}$ & 90.8\pmv{4.5} & 65.2\pmv{33.0} & 0.0\pmv{0.0} & 0.0\pmv{0.0} \\
\texttt{humanoidmaze-large-navigate} & 78.0\pmv{3.8} & 82.0\pmv{7.7} & $\mathbf{87.4\pmv{5.8}}$ & 83.2\pmv{3.2} & 81.6\pmv{6.1} & 82.6\pmv{2.7} & 84.8\pmv{2.7} \\
\bottomrule
\end{tabular}}
\end{table}

\subsection{Categorical Support}\label{app:h5}
The support is set by the rule $[V_{\min},V_{\max}]=[r_{\min}/(1-\gamma),\,r_{\max}/(1-\gamma)]$ rather than tuned per task. Table~\ref{tab:vmin_sweep} varies $V_{\min}$ with the number of atoms fixed. Underestimating the support clips the Bellman targets and breaks training, whereas overestimating it only lowers the resolution and costs little. The practical guideline is simply to make sure the support covers the return range, and to err on the wide side.

\begin{table}[H]
\centering
\footnotesize
\caption{Success rate for different lower bounds $V_{\min}$ of the categorical support ($V_{\max}=0$, $I=101$, five seeds).}
\label{tab:vmin_sweep}
\renewcommand{\arraystretch}{1.1}
\setlength{\tabcolsep}{6pt}
\begin{tabular}{lccccc}
\toprule
\textbf{Task} & $-100$ & $-170$ & $-200$ (rule) & $-230$ & $-300$ \\
\midrule
\texttt{antmaze-large-navigate} & 0.0\pmv{0.0} & 94.4\pmv{3.7}  & 94.6\pmv{4.2} & 93.2\pmv{4.7} & 96.0\pmv{1.4} \\
\texttt{puzzle-3x3-noisy}       & 0.0\pmv{0.0} & 81.1\pmv{22.5} & 98.6\pmv{2.8} & 99.6\pmv{0.8} & 94.2\pmv{11.0} \\
\bottomrule
\end{tabular}
\end{table}

\subsection{Number of Atoms}\label{app:h6}
Table~\ref{tab:atom_sweep} varies the number of atoms $I$ over roughly a tenfold range and reports the training wall-clock time next to each result. Performance is stable once the resolution is sufficient. The only exception is \texttt{scene-play} with $I=51$, where the coarse grid under-resolves the return distribution, and doubling the atom count restores the performance. Wall-clock time barely changes across the range, so the default $I=101$ provides enough resolution at no meaningful cost.

\begin{table}[H]
\centering
\footnotesize
\caption{Success rate and training wall-clock time (hours) for different numbers of atoms $I$ (five seeds).}
\label{tab:atom_sweep}
\renewcommand{\arraystretch}{1.1}
\setlength{\tabcolsep}{5pt}
\resizebox{\linewidth}{!}{%
\begin{tabular}{lcccc}
\toprule
\textbf{Task} & $I=51$ & $I=101$ (default) & $I=201$ & $I=501$ \\
\midrule
\texttt{antmaze-large-navigate}     & 91.6\pmv{4.6} (1.28h)  & 94.6\pmv{4.2} (1.26h) & 92.4\pmv{3.9} (1.23h) & 95.6\pmv{2.3} (1.31h) \\
\texttt{humanoidmaze-giant-navigate} & 79.0\pmv{9.5} (2.28h) & 85.4\pmv{5.0} (2.32h) & 83.5\pmv{6.5} (2.24h) & 79.5\pmv{9.2} (2.26h) \\
\texttt{puzzle-3x3-noisy}           & 98.4\pmv{2.3} (1.22h)  & 98.6\pmv{2.8} (1.19h) & 98.8\pmv{1.6} (1.19h) & 97.2\pmv{3.0} (1.22h) \\
\texttt{scene-play}                 & 20.5\pmv{18.6} (1.55h) & 92.4\pmv{4.5} (1.57h) & 86.8\pmv{5.9} (1.58h) & 86.8\pmv{4.3} (1.57h) \\
\bottomrule
\end{tabular}}
\end{table}

\section{Broader Impact and Limitation}\label{app:limit}

\paragraph{Broader Impact}
This work improves the reliability of offline reinforcement learning in sparse and long-horizon environments by reducing sensitivity to high-variance multi-step targets. Such improvements may benefit applications where online exploration is expensive or unsafe, including robotics, autonomous navigation~\cite{10969773,kang2025offlineramp,11080103,sangeunksae}, UAV control~\cite{eo2024offline}, and decision-support systems trained from logged data~\cite{lee2026ai,lee2025deep}.

At the same time, stronger offline RL methods could enable automated decision-making systems with greater real-world impact. If trained on biased or incomplete datasets, such systems may inherit or amplify undesirable behaviors present in the data. Careful dataset curation, evaluation, and deployment safeguards therefore remain important when applying offline RL methods in safety-critical settings.

\paragraph{Limitation}
The proposed approach introduces modest additional computation due to the use of rejection sampling and distributional aggregation. However, these operations are fully parallelizable and align well with modern accelerator hardware, making the practical overhead small in typical training settings. 

Another limitation is that the method relies on categorical distributional critics to represent return distributions. While such critics are widely used in distributional RL, extending the proposed aggregation operator to other distributional representations (e.g., quantile-based critics) remains an interesting direction for future work.

Finally, our reliability hypothesis targets return variance that comes from heterogeneous behavior in the dataset. Most of the benchmarks we use have deterministic dynamics, so this is the dominant source of variance there. In environments with strongly stochastic dynamics, part of the return variance is irreducible, and penalizing it indiscriminately can make the policy more conservative than necessary even for well-supported actions. A state-dependent $\delta$ that relaxes the aversion where the remaining variance is unavoidable, or a $\delta$ that is annealed during online fine-tuning, would be natural extensions. Both fit the framework without modification, but estimating such a coefficient reliably needs additional uncertainty machinery, and we leave this to future work.

\end{document}